%% file: main.tex
\documentclass{article} 
\usepackage{ICLR2025/iclr2025_conference,times}

\input{ICLR2025/math_commands.tex}

\input{macro}
\usepackage{hyperref}
\usepackage{subcaption}
\usepackage{url}

\title{Phase-aware Training Schedule Simplifies Learning in Flow-Based Generative Models}

\author{Santiago Aranguri \thanks{Equal contribution.}\\
Courant Institute of \\Mathematical Sciences\\
New York University\\
New York, NY 10012, USA \\
\texttt{aranguri@nyu.edu} \\
\And
Francesco Insulla \footnotemark[1] \\
Institute of Computational\\ and Mathematical Engineering\\
Stanford University\\
Stanford, CA 94305, USA \\
\texttt{franinsu@stanford.edu}
}

\iclrfinalcopy 
\begin{document}

\maketitle

\begin{abstract}
We analyze the training of a two-layer autoencoder used to parameterize a flow-based generative model for sampling from a high-dimensional Gaussian mixture. 
Previous work shows that the phase where the relative probability between the modes is learned disappears as the dimension goes to infinity without an appropriate time schedule. We introduce a time dilation that solves this problem. This enables us to characterize the learned velocity field, finding a first phase where the probability of each mode is learned and a second phase where the variance of each mode is learned. We find that the autoencoder representing the velocity field learns to simplify by estimating only the parameters relevant to each phase. Turning to real data, we propose a method that, for a given feature, finds intervals of time where training improves accuracy the most on that feature. Since practitioners take a uniform distribution over training times, our method enables more efficient training. We provide preliminary experiments validating this approach.





\end{abstract}

\section{Introduction}
In recent years, diffusion models have emerged as a powerful technique for learning to sample from high-dimensional distributions \cite{pmlr-v37-sohl-dickstein15, song2021scorebased, song2020generativemodelingestimatinggradients, ho2020denoisingdiffusionprobabilisticmodels}, especially in the context of generating images and recently also for text \cite{lou2024discretediffusionmodelingestimating}. The idea lies in learning, from data samples, a velocity field that pushes noisy datapoints to clean datapoints. Despite the remarkable performance of these models, there remain several open questions, including understanding what makes a good noise schedule, which is the focus of this paper.

We consider the problem of training a neural network to learn the velocity field to generate samples from a two-mode Gaussian mixture (GM). This serves as a prototypical example to understand how diffusion models handle learning features at different scales, since the two-mode GM has two scales: the macroscopic scale of the probability of each mode, and the microscopic scale of the variance of each mode. 

This problem was previously considered by \cite{cui2024analysislearningflowbasedgenerative}, but their analysis only handles the \textit{balanced} two-mode GM (i.e. the probability of each mode is exactly $1/2.$) On the other hand,  \cite{biroli2024dynamicalregimesdiffusionmodels} assume access to the exact velocity field and find that the phase where the probability of each mode is learned disappears as the dimension of the problem grows.

In this work, we first introduce a noise schedule that makes the phase where this probability is learned not disappear as the dimension goes to infinity. This enables us to extend the analysis of \cite{cui2024analysislearningflowbasedgenerative} to the two-mode GM without the balanced assumption. More precisely, our contributions are as follows.

\begin{itemize}
    \item We give an asymptotic characterization of the learned velocity field for learning to generate the two-mode GM, finding a separation into two phases. We further show that $\Theta_d(1)$ samples are sufficient to learn the velocity field.
    \item We show that the neural network representing the velocity field learns to simplify for each phase. In the first phase, it only concerns estimation of the probability of each mode, whereas in the second phase, it concerns estimation of the variance of each mode. This sheds light on the advantage of diffusion models over denoising autoencoders, since the sequential nature of diffusion models shown here allows them to decompose the complexity of the problem.
    \item We show that the phase transition separating the two phases can be detected from a discontinuity in the Mean Squared Error associated to the learning problem, which suggests a way to find these transitions for general data distributions.
    \item For real data, this analysis suggests that training more at the times associated with a feature improves accuracy on that feature. In fact, we propose a method that, given a feature, finds an interval of time where more training improves accuracy on that feature the most. We further validate this on the MNIST dataset. We provide the code for the experiments \href{https://github.com/submissioniclr12263/submission}{here}.
\end{itemize}

\section{Related Works}
\textbf{Phase transitions of generative models in high dimensions.} Several works analyze phase transitions in the dynamics of generative models. \cite{raya2023spontaneoussymmetrybreakinggenerative} find that diffusion models can exhibit symmetry breaking, where two phases are separated by a time where the potential governing the dynamics has an unstable fixed point. They give a full theoretical analysis for the data being two equiprobable point masses in $\mathbb{R}$, and also give a bound for the symmetry breaking time for the case where the data is a sum of finitely many point masses. Our setting generalizes the case of two equiprobable point masses in $\mathbb{R}$ to two Gaussians in $\mathbb{R}^{d}$ that are not necessarily equiprobable. \cite{ambrogioni2023statistical} builds on \cite{raya2023spontaneoussymmetrybreakinggenerative} and shows several connections between equilibrium statistical mechanics and the phase transitions of diffusion models. \cite{ambrogioni2023statistical} further conjectures that accurately sampling near times of "critical generative instability" affects the sample diversity. We give an explicit description of these critical times and verify this conjecture theoretically for sampling (see Proposition 1) and for learning (see Corollary 5) and empirically for learning (see Section 6). \cite{li2024critical} also formalize the study of critical windows taking the data to be a mixture of strongly log-concave densities. They give non-asymptotic bounds for the start and end times of these critical windows, which have a closed form expression for the mixtures of isotropic Gaussians case. In contrast, we provide sharp asymptotic characterizations for the phase transition times. \cite{Biroli_2023} analyze the Curie-Weiss model and analytically characterize the speciation time, defined as the time after which the mode that the sample will belong to is determined. \cite{biroli2024dynamicalregimesdiffusionmodels} generalize the result and find an speciation time $t_s \sim \frac{1}{2}\log (\lambda)$ for an Ornstein-Uhlenbeck process where $\lambda$ is the largest eigenvalue of the covariance of the data, usually proportional to $d$. 
\cite{montanari2023samplingdiffusionsstochasticlocalization} points out a similar phase transition when \textit{learning} the velocity field to generate from a two-mode unbalanced Gaussian mixture, leading to problems for accurate estimation of the data. 
\cite{montanari2023samplingdiffusionsstochasticlocalization} addresses this by using a different neural network to learn each mode. In the current work, we show that it is not necessary to tailor the network for each mode if the right time schedule is used. It is worth noting that all these works are about sampling. We provide a result for sampling in Proposition \ref{prp:char:gm}. Building on this, we give results for learning (i.e. estimating the velocity field through a neural network) which is the main contribution of our paper.


\textbf{Time-step complexity.} Several results give convergence bounds detailing the required time-steps, score accuracy, and/or data distribution regularity to sample accurately. \cite{benton2024nearlydlinearconvergencebounds} show that at most $O(d\log^2(1/\delta)/\epsilon^2)$ time steps are required to approximate a distribution corrupted with Gaussian noise of variance $\delta$ to within $\epsilon^2$ KL divergence. \cite{chen2023probabilityflowodeprovably} study probability flow ODE and obtain $O(\sqrt{d})$ convergence guarantees with a smoothness assumption. An underlying assumption in all these works is that the score or velocity field is learned to a certain accuracy. In the present work, we address this problem in the special case of a Gaussian mixture.

\textbf{Sample complexity for Gaussian mixtures.} \cite{cui2024analysislearningflowbasedgenerative} study the learning problem for the Gaussian mixture in high dimensions and demonstrate that $n=\Theta_d(1)$ samples are sufficient in the balanced case where the two modes have the same probability. This is done through statistical physics techniques of computing the partition function and using a sample symmetric ansatz. As we show, due to the speciation time at $d^{-1/2}$ which tends to zero as the dimension $d$ grows, this analysis misses one phase of learning. \cite{gatmiry2024learningmixturesgaussiansusing} show that quasi-polynomial ($O(d^{\text{poly}(\log(\frac{d+k}{\epsilon}))})$) sample and time complexity is enough for learning $k$-gaussian mixtures. The data distribution is more general than the one we consider, but on the other hand we give a $\Theta_d(1)$ sample and time complexity. 

\section{Background}
\label{sec:back}
\textbf{Data and flow-based generative model.} Consider the two-mode Gaussian mixture (GM)
\begin{equation}
\rho=p\mathcal{N}(\mu,\sigma^2\text{Id}_{d})+(1-p)\mathcal{N}(-\mu,\sigma^2\text{Id}_{d})\label{eq:gmmm:2m}
\end{equation}
where $p\in(0,1)$ and $\mu\in\mathbb{R}^{d}$ such that $\|\mu\|^{2}=d$ and $\sigma=\Theta_d(1)$. A diffusion model for $\rho$ starts with samples from a simple distribution (say a Gaussian) and sequentially denoises them to get samples from the data. More precisely, consider the stochastic interpolant
\begin{equation}
x_{t}=\alpha_{t}x_0+\beta_{t}x_1\label{eq:gen_int}
\end{equation}
where $x_0\sim\mathcal{N}(0,\text{Id}_{d}),$ $x_1\sim\rho,$ and $\alpha_{t},\beta_{t}:[0,1]\to\mathbb{R},$ $\alpha_0=1=\beta_{1},$ $\alpha_{1}=0=\beta_{0}.$ Stochastic interpolants are introduced in \cite{albergo2023stochasticinterpolantsunifyingframework}, and they prove that if $X_{t}$ solves the probability flow ODE 
\begin{equation}
    \Dot{X}_{t}=b_{t}(X_{t})\qquad\text{with}\qquad b_{t}(x)=\mathbb{E}[\Dot{x}_{t}|x_{t}=x]\label{eq:ode:gen:1}
\end{equation}
with $X_{0}\sim \mathcal{N}(0,\text{Id}_{d})$, we then have $X_{t}\stackrel{d}{=}x_{t}$ for $t\in[0,1]$ and hence $X_{t=1}\sim\rho.$ We call $X_t$ the flow-based generative model associated to the interpolant $I_t.$

Since $\rho$ is a Gaussian mixture, the expression for the exact velocity field $b_t(x)$ from \eqref{eq:ode:gen:1} can be computed exactly. Our goal is to understand how well a neural network can estimate this velocity field through samples, in the large dimension $d\to\infty$ limit assuming low sample complexity for the data $n=\Theta_d(1).$

\textbf{Loss function.} To fulfill our goal, we rewrite the velocity field as 
\begin{equation}
    \label{eq:dec:b}
    b_t(x) = \left(\dot\beta_t-\frac{\dot\alpha_t}{\alpha_t}\beta_t\right) f(x, t) + \frac{\dot\alpha_t}{\alpha_t}x,
\end{equation}
where $f(x,t)=\mathbb{E}[x_1|x_t=x]$ is called the denoiser since it recovers the datapoint $x_1$ from a noisy version $x_t.$ The denoiser is characterized as the minimizer of the loss (see \cite{albergo2023stochasticinterpolantsunifyingframework})
\begin{equation}
    \mathcal{R}[f] = \int_0^1 \mathbb{E}||f(x_t, t) - x_1||^2dt.
    \label{eq:obj}
\end{equation}
In practice, however, we usually do not have access to the exact data distribution. So we assume we have a dataset $\mathcal{D}=\{x_1^\mu\}^n_{\mu=1}$ where $x_1^\mu \sim_\text{iid} \rho.$ On the other hand, we have unlimited samples from $x_0\sim \mathcal{N}(0,\text{Id}_{d}).$ Hence, to each data sample $x_1^\mu$ we can associate several noise samples $x_0^{\mu,\nu}$ with $\nu=1,\cdots,k$. We then denote $x_t^{\mu,\nu}=\alpha_t x_0^{\mu,\nu} +\beta_t x_1^\mu.$ Later in our analysis, we will assume infinitely many noise samples associated to each data sample, so that we can take expectation with respect to the noise distribution. 

We parameterize the denoiser with a single neural network for each $t,$ which we denote as $f_{\theta_t}(x).$ We get then an empirical version of the loss in equation \ref{eq:obj} $\mathcal{\hat{R}}(\{\theta_t\}_{t\in[0,1]}) = \int_0^1 \mathcal{\hat{R}}_t(\theta_t)dt$ where 
\begin{equation}
    \mathcal{\hat{R}}_t(\theta_t) = \sum^n_{\mu=1}\sum^k_{\nu=1}||f_{\theta_t}(x_t^{\mu,\nu}) - x_1^\mu||^2
    \label{eq:emp:loss1}
\end{equation}

\textbf{Network architecture.} We focus on the case where the neural network parameterizing the denoiser function $f(x, t)$ is a two-layer denoising autoencoder with a trainable skip connection as follows
\begin{equation}
    \label{eq:dae}
    f_{\theta_t}(x) = c_tx + u_t\tanh\left(\frac{w_t\cdot x}{\sqrt{d}} + b_t\right)
\end{equation}
where $\theta_t = \{c_t, u_t,w_t, b_t\};$ $c_t,b_t\in \mathbb{R};$ and $u_t, w_t \in \mathbb{R}^d.$ The structure of this denoising autoencoder is a particular case of the U-Net from \cite{DBLP:journals/corr/RonnebergerFB15} and is motivated by the exact denoiser which can be computed exactly since the data distribution is a Gaussian mixture
\begin{equation}
    \label{eq:exact:b}
    \mathbb E[x_1|x_t=x] = \frac{\beta_t\sigma^2}{\alpha_t^2 + \sigma^2\beta_t^2} x + \frac{\alpha_t^2}{\alpha_t^2 + \sigma^2\beta_t^2}\mu \tanh\left(\frac{\beta_t}{\alpha_t^2 + \sigma^2\beta_t^2} \mu\cdot x + h\right)
\end{equation}
where $h$ is such that $e^h/(e^h+e^{-h})=p.$ (See \cite{albergo2023stochasticinterpolantsunifyingframework}, Appendix A for the proof.)

We add to the loss regularization terms for $w_t$ and $u_t,$ giving
\begin{equation}
    \mathcal{\hat{R}}_t(\theta_t) = \sum^n_{\mu=1}\sum^k_{\nu=1}||f_{\theta_t}(x_t^{\mu,\nu}) - x_1^\mu||^2 + \frac{\lambda}{2} ||u_t||^2 + \frac{\ell}{2} ||w_t||^2
    \label{eq:emp:loss}
\end{equation}
Denoting $\hat \theta_t$ the minimizer of this loss, we define
\begin{align}
    \hat b_t(x) = \left(\dot\beta_t-\frac{\dot\alpha_t}{\alpha_t}\beta_t\right) f_{\hat\theta_t}(x) + \frac{\dot\alpha_t}{\alpha_t}x.
    \label{eq:hat:vel:field}
\end{align}
Using this velocity field, we then run the probability flow ODE 
\begin{align}
    \label{eq:ode:emp}
    \Dot{\hat X}_{t}=\hat b_{t}(\hat X_{t});\quad \hat X_{0}\sim \mathcal{N}(0,\text{Id}_{d}).
\end{align}
Our goal is to understand how close $\hat{X}_2$ is to a sample from the Gaussian mixture $\rho.$

\cite{cui2024analysislearningflowbasedgenerative} consider the special case of tied weights $u_t=w_t$ and $b_t=0.$ This is enough to learn to sample from the balanced two-mode GM (i.e. $p=1/2$) but fails at the two-mode GM for $p\neq 1/2.$ This follows because $x_0$ has an even distribution and their choice of tied weights and no bias yields an odd velocity field which results in an even distribution for $x_t$. If the weights are untied and the bias is added, the analysis of \cite{cui2024analysislearningflowbasedgenerative} still does not work to show that $\hat X_1$ has the correct $p$ for $p\neq 1/2.$ This is because the gradients for $w_t$ and $b_t$ vanish as $d\to\infty$ unless special care is given to the small times where a phase transition related to learning the probability between the modes occurs, as will be explained next.

\textbf{Separation into phases.} \cite{biroli2024dynamicalregimesdiffusionmodels} show that the generative model with the exact velocity field from \eqref{eq:ode:gen:1} with $\alpha_t=\sqrt{1-t^2}$ and $\beta_t = t$ undergoes a phase transition at the speciation time $t_s=1/\sqrt{d}.$ The speciation time is defined as the time in the generation process after which the mode that the sample will belong to at the end of the process is determined. Their analysis can be extended to show that the speciation time is still $t_s=1/\sqrt{d}$ if we instead have $\alpha_t=1-t$ and $\beta_t = t$ which are the choices in our paper. Since this result is only mentioned as motivation, we will not prove it. 

The analysis of \cite{cui2024analysislearningflowbasedgenerative} relies on taking the $d\to\infty$ limit and obtaining a limiting ODE. Since $t_s=1/\sqrt{d}$ goes to zero as $d\to\infty,$ their limiting ODE has a singularity at $t=0$ and the possibility of learning the probability of each mode is lost. This is in essence why the analysis of \cite{cui2024analysislearningflowbasedgenerative} can not capture the learning of $p$ for $p\neq 1/2.$ 

We will dilate time so as to make the speciation time $t_s$ not disappear as $d\to \infty.$ More precisely, we define
\begin{align}
    \label{eq:time_dil}
\tau(t) = \begin{cases}
    \frac{\kappa t}{\sqrt{d}} & \text{if } t \in [0,1]\\
    \frac{\kappa}{\sqrt{d}} + \left(1-\frac{\kappa}{\sqrt{d}}\right)(t-1) & \text{if } t \in [1,2].
\end{cases}
\end{align}
This fulfills $\tau(0) = 0, \tau(1) = \kappa/\sqrt{d},$ and $\tau(2)=1.$ We prove next that the generative model from \eqref{eq:ode:gen:1} with $\alpha_t=1-\tau_t$ and $\beta_t=\tau_t$ has two phases: for $t\in[0,1]$ the probability of each mode is estimated, and for $t\in [1,2]$ the variance of each mode is estimated.
\begin{prop}
    \label{prp:char:gm}
    Let $X_t$ be the solution to the probability flow ODE  from \eqref{eq:ode:gen:1} with $\alpha_t=1-\tau_t$ and $\beta_t = \tau_t$ where $\tau_t$ is defined in \eqref{eq:time_dil}. Then for $t\in[0,2]$ we have
    \begin{align*}
        X_t - \frac{\mu\cdot X_t}{{d}}\mu\sim \mathcal{N}\left(0, \sigma_t^2\text{Id}_{d-1}\right).
    \end{align*}
    where $\sigma_t$ is characterized below. We further have the following phases
    \begin{itemize}
        \item 
            \textbf{First phase}: For $t\in [0,1],$ we have $\lim_{d \to\infty} \sigma_t =1.$
            
            In addition, $\nu_t = \lim_{d \to\infty} \frac{\mu\cdot X_t}{\sqrt{d}}$ fulfills 
            \begin{align*}
                \nu_1\sim p\mathcal{N}(\kappa, 1) + (1-p)\mathcal{N}(-\kappa, 1).
            \end{align*}
        \item \textbf{Second phase}: 
            We have $\lim_{d \to\infty} \sigma_2=\sigma.$
            
            In addition, $M_t = \lim_{d \to\infty}  \frac{\mu\cdot X_t}{{d}}$ fulfills
            \begin{align*}
                M_2 \sim p_\kappa\delta_1 + (1-p_\kappa)\delta_{-1}
            \end{align*}
            where $p^\kappa$ is such that  $\lim_{\kappa\to \infty}p_\kappa=p$
    \end{itemize}
\end{prop}
See Appendix \ref{app:prf} for the proof of this Proposition. In Appendix \ref{app:gen:dil}, we give a generalization of the time dilation formula in equation \ref{eq:time_dil} for a Gaussian mixture with more than two modes.

Without the time dilation, we can not capture the learning of $p$ for $p\neq 1/2$ because the first phase (where this parameter is learned) disappears as $d\to\infty.$ The time dilation will allow us to analyze the phase where $p$ is learned in the $d\to\infty$ limit and hence show that $\hat{X}_2$ recovers $p.$

We show this in two steps. In Section \ref{sec:l}, we characterize the learned parameters of the velocity field in terms of a few projections, called the overlaps. Then, in Section \ref{sec:gen}, we combine these characterizations with Proposition \ref{prp:char:gm} to show that $\hat{X}_2$ recovers the parameters $p$ and $\sigma^2$ of the two-mode Gaussian mixture $\rho$ under appropriate limits.
\section{Learning}
\label{sec:l}
In this section, we will characterize $\hat \theta_t,$ the minimizer of the loss from \eqref{eq:emp:loss}, which is used to parameterize the velocity field that yields $\hat{X}_t$ (see \eqref{eq:ode:emp}.) We take $\alpha_t=1-\tau_t$ and $\beta_t=\tau_t$ and analyze $\hat \theta_t$ in the $d\to\infty$ limit. We first analyze the times $t\in[0,1]$ and then $t\in [1,2].$

\subsection{First Phase}
The interpolant from \eqref{eq:gen_int} in the first phase reads
\[
x_t^\mu = \left(1 - \frac{\kappa t}{\sqrt{d}}\right)x_0^\mu + \frac{\kappa t}{\sqrt{d}} x_1^\mu
\]
where $t\in\left[0,1\right].$
To characterize $\hat{\theta}_t= \{c_t, u_t,w_t, b_t\},$ we introduce the following overlaps (dropping the dependence on $t$ for notational simplicity.)
\begin{align}
\label{first:ov}
p_\eta^\mu = \frac{z^\mu\cdot w}{d}\;\,\,
\omega = \frac{\mu\cdot w}{d} \;\,\,
r = \frac{\|w\|^2}{d}\;\,\,
q_\xi^\mu = \frac{x_0^\mu\cdot u}{d}\;\,\,
q_\eta^\mu = \frac{z^\mu\cdot u}{d} \;\,\,
m = \frac{\mu\cdot u}{d} \;\,\,
q = \frac{\|u\|^2}{d}.
\end{align}

We now give equations for the overlaps in the asymptotic $d\to\infty$ limit.

\begin{res}[Sharp Characterization of Parameters in First Phase] \label{res:first} For any $t\in[0,1]$, the overlaps associated to $\hat{\theta}_t,$ the minimizer of the loss from \eqref{eq:emp:loss}, satisfy the following in the $d\to\infty$ limit
\label{res1}
\begin{minipage}[t]{0.49\textwidth}
    \begin{align*}
        m &= \frac{n \overline{\phi s}}{\lambda + n \overline{\phi^2}}\\
        q_\eta &= \frac{\sigma}{n} m
    \end{align*}
\end{minipage}%
\hfill
\begin{minipage}[t]{0.49\textwidth}
    \begin{align*}
        c &= q_\xi = p_\eta = 0\\
        q &= m^2 + n q_\eta^2 \\
        r &= \omega^2\\
    \end{align*}
\end{minipage}
\vspace{-.4cm}
\begin{align*}
    &(\lambda + n\overline{\phi^2})\overline{\phi's}=n\overline{\phi s}\, \overline{\phi \phi'}\\
    &    \hat r (\lambda + n\overline{\phi^2})^2 
    = - n((\lambda + n\overline{\phi^2})(\sigma^2 + n) \overline{\phi'' s} \, \overline{\phi s}  - n(\sigma^2 +  n)\overline{\phi s}^2\,\overline{(\phi\phi')'})\\
    &\omega(\ell + \hat r)(\lambda + n\overline{\phi^2})^2 = (n\kappa t)(\sigma^2  + n)((\lambda + n\overline{\phi^2})(\overline{\phi' } \, \overline{\phi s}) - n \overline{\phi s}^2 \overline{\phi'\phi s} )
\end{align*}

Here and in what follows, we denote \[
\overline{y}=\frac{1}{nk}\sum_{\mu=1}^n\sum_{\nu=1}^k \E_{z^{\mu,\nu}}[y^{\mu,\nu}] = \overline{p}  \E_{z^{\mu,\nu}}[y^{\mu,\nu}|s^\mu=1] +(1-\overline{p})  \E_{z^{\mu,\nu}}[y^{\mu,\nu}|s^\mu=-1].
\]
\end{res}

See Appendix \ref{sec:first_phase} for a heuristic derivation of this result, at the level of rigor of theoretical physics. We next show that the equations for the overlaps simplify in the $n\to\infty$ limit.

\begin{cor}[Parameters given infinite samples] For any $t\in[0,1]$, taking $d\to\infty$ and then $n\to\infty$ gives the following overlaps
\label{cor:sta}

\vspace{-0.6cm}
\noindent
\begin{minipage}[t]{0.49\textwidth}
    \begin{align*}
        &\tanh(b) = 2\left(p-\tfrac{1}{2}\right),\\
        &c = q_\xi = q_\eta = p_\eta = 0,
    \end{align*}
\end{minipage}%
\hfill
\begin{minipage}[t]{0.49\textwidth}
    \begin{align*}
        m &= 1,\\
        \omega &= \kappa t.
    \end{align*}
\end{minipage}
\end{cor}

See Appendix \ref{sec:cor1} for the derivation. Note that the overlaps in the $n\to\infty$ limit do not contain any information about $\sigma^2,$ showing that the estimation of $\sigma^2$ happens completely in the second phase. 

We now turn to the Mean Squared Error. Define the scaled train and test MSE of the denoiser as 
\[
{\text{mse}_\text{train}} = \frac{1}{dnk}\sum^n_{\mu=1}\sum^k_{\nu=1}||f_{\theta_t}(x_t^{\mu,\nu}) - x_1^\mu||^2\quad\quad
{\text{mse}_\text{test}} = \frac{1}{d}\E\left[\|f_{\hat \theta_t}(x_t)-x_1\|^2\right].
\]
Using the above results we characterize the MSE

\begin{cor} In the limit of $d\to\infty,$
\label{first:mse}
\begin{align*}
    \text{mse}_{\text{train}} &= 1+\sigma^2 + c^2 + q\overline{\phi^2} - 2\overline{s\phi}(  m + \sigma q_\eta - c q_\xi)\\
    \text{mse}_{\text{test}} &= 1+\sigma^2 + c^2 + q\overline{\phi^2} - 2 \overline{s\phi} m
\end{align*}
For $n\to \infty,$ we get
\[
\text{mse}_{\text{train}} = \text{mse}_{\text{test}} = \sigma^2 + (1 - \overline{\phi s}).
\]
\end{cor}

\subsection{Second Phase}
We now consider times $t\in \left[1,2\right]$ which means we have 
$$x_t^\mu = (2-t)\left(1-\frac{\kappa}{\sqrt{d}}\right)x_0^\mu + \left(\frac{\kappa}{\sqrt{d}} + \left(1-\frac{\kappa}{\sqrt{d}}\right)\left(t-1\right)\right)x_1^\mu.$$

Using the same definitions of overlaps as for the first phase, we find closed-form equations for the overlaps in the asymptotic $d\to\infty$ limit, and again find the limit as $n\to\infty$ for the overlaps. See Appendix \ref{sec:res2} for a heuristic derivation of this result

\begin{res}[Sharp Characterization of Parameters in Second Phase]\label{res:second} 
    For any $t\in[1,2]$, in the $d\to\infty$ limit, the parameters minimizing the loss from \eqref{eq:emp:loss} satisfy the following equations
    \begin{minipage}[t]{0.49\textwidth}
        \begin{align*}
            q_\xi  &= \frac{c(1-\tau)}{\lambda + n} \\
            q_\eta &= \frac{\sigma(1-c\tau)}{\lambda + n} \\
        \end{align*}
    \end{minipage}
    \begin{minipage}[t]{0.49\textwidth}
        \begin{align*}
            m &=\frac{n(1-c\tau)}{\lambda + n} \\
            q&= m^2 + nq_\xi^2 + n\sigma^2 q_\eta^2 \\
        \end{align*}
    \end{minipage}
    \vspace{-.4cm}
    \begin{align*}
        c= \frac{\tau\left((1+\sigma^2) (\lambda + n) -  (\sigma+n)\right)}{ (\lambda+n)((1-\tau^2)+(1+\sigma^2)\tau^2)+\left((1-\tau)^2-\tau^2(\sigma+n)\right)}
    \end{align*}
    where $\tau = t-1$.
\end{res}
\begin{cor}[Parameters given inifite samples] For any $t\in[1,2]$, taking $d\to\infty$ and then $n\to\infty$ gives the following overlaps
\label{cor:3}
\begin{align*}
    &c = \frac{\tau\sigma^2}{1+(\sigma^2-1)\tau^2}\quad \quad\quad q_\xi = q_\eta = 0\quad \quad \quad m = 1 - c\tau
\end{align*}
where $\tau = t-1$.
\end{cor}

In contrast to the first phase, the parameter $p$ does not appear in the overlaps whereas now $\sigma^2$ does. Hence, combining Corollaries \ref{cor:sta} and \ref{cor:3} shows that the separation into phases can be learned by the generative model.

We also obtain the MSE for the second phase

\begin{cor} In the limit of $d\to\infty,$ we have 
    \label{sec:mse}
\begin{align*}
    \text{mse}_{\text{train}} &= (1+\sigma^2)(1-c\tau)^2 + c^2(1-\tau)^2 + q -2(1-c\tau)(\sigma q_\eta + m) + 2 c(1-\tau)q_\xi  \\
    \text{mse}_{\text{test}}  &= (1+\sigma^2)(1-c\tau)^2 + c^2(1-\tau)^2 + q -2(1-c\tau) m 
\end{align*}

For $n\to\infty,$ we get
\begin{align*}
    \text{mse}_{\text{train}} = \text{mse}_{\text{test}} =  \sigma^2(1-c\tau)^2 + c^2(1-\tau)^2.
\end{align*}
where $\tau=t-1.$
\end{cor}

In Appendix \ref{sec:res2} we show that combining Corollaries \ref{first:mse} and \ref{sec:mse} gives

\begin{cor}
    \label{cor:both:mse}
    Taking $d\to\infty$ and then $n\to\infty$ gives
    \begin{align*}
        \text{mse}_{\text{test}} = \begin{cases}
        \sigma^2 + 4p(1-p) &\text{ if}\,\,\,\, t=0\\
        \sigma^2 + (1-\overline{\phi^2}) &\text{ if}\,\,\,\, t\in (0, 1)\\
        \sigma^2 &\text{ if}\,\,\,\, t=1^+\\
        0 &\text{ if}\,\,\,\, t=2
    \end{cases}
    \end{align*}
\end{cor}

If we had not dilated time, in the limit of $d\to\infty,$ the $\text{mse}_\text{test}$ would have a jump from $\sigma^2+4p(1-p)$ at $t=0$ to $\sigma^2$ at $t=0^+.$ By dilating time, we make a transition between these two values with $\text{mse}_\text{test}=\sigma^2 + (1-\overline{\phi^2})$ for $0<t<1,$ where $1-\overline{\phi^2}=4p(1-p)$ for $t=0$ and $1-\overline{\phi^2}$ goes to $0$ exponentially fast as $\kappa$ grows when $t=1.$ Further, $\overline{\phi^2}$ is continuous in $t\in[0,1].$ Hence by dilating near the phase transition, we decreased the jump discontinuity of the \emph{mse}.

Remarkably, for generating samples from a general data distribution, this result suggests that the jumps in the \emph{mse} could correspond to phase transitions. Further, this phase transitions could be resolved by dilating near the jump in the \emph{mse}. We leave the study of this conjecture to future work.




\section{Generation}\label{sec:gen}
Having characterized the parameters $\hat{\theta}_t,$ we now show that $\hat{X}_2$ has the right parameters $p$ and $\sigma^2$ from the data distribution $\rho.$ Let $X_t$ be the solution to the ODE from \eqref{eq:ode:gen:1} using the exact denoiser from \eqref{eq:exact:b}. Assume $X_{t}$ and $\hat{X}_t$ have a shared initial condition $X_{t=0}=\hat{X}_{t=0}\sim\mathcal{N}(0,\text{Id}_d)$. Then $X_t-\hat{X}_t$ fulfills an ODE with initial condition $0$ whose velocity field is in the span of $u_t$ and $\mu.$

Result \ref{res:first} gives that in the first phase $q=m^2+nq_\eta^2.$ This can be explicitly stated as
$$
\lim_{d\to\infty} \frac{\|u\|^2}{d} = \lim_{d\to\infty}\left(\frac{\mu\cdot u}{d}\right)^2 + \left(\frac{\eta\cdot u}{d}\right)^2
$$
where $\eta =\sigma  \sum_{\mu=1}^n z^\mu.$ This means that $u_t$ is asymptotically contained in span$(\mu, \eta),$ in the sense that the projection to the complement of span$(\mu, \eta)$ has asymptotically vanishing norm, for $t\in[0,1].$ Similarly, from Result \ref{res:second}, we get $q= m^2 + nq_\xi^2 + n q_\eta^2,$ which means that $u_t$ is asymptotically contained in span$(\mu, \eta, \xi)$ for $t\in[1,2]$ where $\xi =\sum_\mu s^\mu x_0^\mu.$ This means that to show that $X_t$ is close to $\hat X_t,$ it suffices to bound the projections of $X_t-\hat{X}_t$ onto $\mu$, $\eta$, and $\xi.$ In fact, we have the following result (see Appendix \ref{gen:app})

\begin{res}
    \label{res:exact:vs:learned}
    Let $X_t$ be the solution of the probability flow ODE from \eqref{eq:ode:gen:1} using the exact denoiser from \eqref{eq:exact:b}. Let $\hat{X}_t$ be the solution using the learned denoiser. Assume $X_{t=0}=\hat{X}_{t=0}\sim\mathcal{N}(0,\text{Id}_d)$. Then for $w\in\text{span}(\mu,\eta,\xi),$ with $\|w\|_2=1$, we have
    \begin{align*}
        \lim_{d\to\infty}\frac{w \cdot (X_{2} - \hat X_{2})}{\sqrt{d}} = O\left(\frac{1}{n}\right).
    \end{align*}
    For $w\in\text{span}(\mu,\eta,\xi)^\perp,$ with $\|w\|_2=1,$ we have
    \begin{align*}
        \lim_{d\to\infty}\frac{w \cdot (X_{2} - \hat X_{2})}{\sqrt{d}} = 0.
    \end{align*}
\end{res}

\begin{cor}[Parameters $p$ and $\sigma^2$ are estimated correctly]
    \label{cor:5}
    Let $\hat{X}_t$ be the solution of the probability flow ODE from \eqref{eq:ode:gen:1} using the learned denoiser, starting from $\hat{X}_0\sim \mathcal{N}(0,\text{Id}_d).$ We have
    \begin{align*} 
        \lim_{\kappa\to\infty} \lim_{n\to\infty} \lim_{d\to\infty} \frac{\mu\cdot \hat{X}_2}{d}\sim p\delta_1+(1-p)\delta_{-1}.
    \end{align*}
    For $w \perp \mu,$ with $\|w\|_2=1,$ we have
    \begin{align*}
        \lim_{n\to\infty} \lim_{d\to\infty} \frac{w\cdot \hat{X}_2}{\sqrt{d}}\sim \mathcal{N}(0,\sigma^2).
    \end{align*}
\end{cor}
We conclude that the distribution generated using the learned denoiser captures both $p$ and $\sigma^2.$

\section{Experiments}
\subsection{Verification that parameter $p$ is captured}
To demonstrate the difference between the time dilated and non-dilated interpolants in practice we construct the following simple experiment. We run Gradient Descent with the Adam optimizer \cite{diederik2014adam} to learn the parameters $w_t,c_t,u_t,b_t$ in \eqref{eq:dae} both for $\alpha_t=1-t,\beta_t=t$ and the dilated version $\alpha_t=1-\tau_t,\beta_t=\tau_t.$ The results are shown in Figure \ref{fig:enter-label} and suggest time-dilation is required to estimate the probability of each mode.


The code for this experiment is available \href{https://github.com/submissioniclr12263/submission}{here}. 
\begin{figure}
    \centering
    \includegraphics[width=0.6\linewidth]{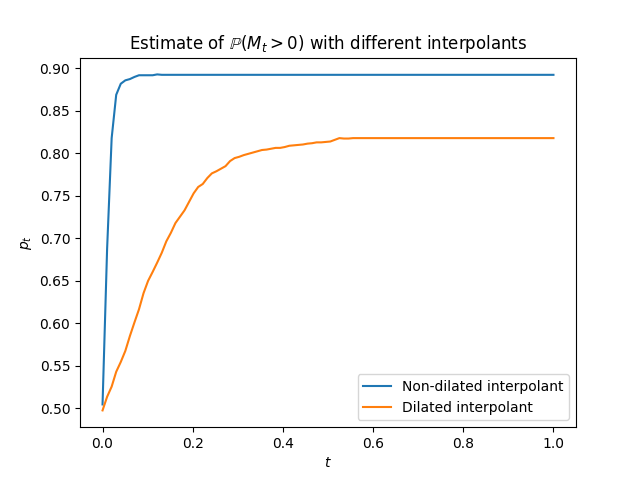}
    \caption{We learn the parameters from \eqref{eq:dae} for different choices of interpolant. In all experiments, we take $100$ discretization points, train for $5000$ epochs, with $n=128,$ $d=5000,$ and $p=.8.$ We then run the probability flow ODE with the learned parameters for $K=2000$ realizations and estimate $\mathbb{P}(M_t>0)=p$ with $M_t=\mu\cdot X_t/d.$ For the non-dilated interpolant in blue, we use $\alpha_t=1-t,\beta_t=t$. We predict the speciation to happen near $t=1/\sqrt{5000}\approx .014$ as confirmed by the experiment since most of the speciation occurs at the first two ODE steps. For the dilated interpolant in orange, we use $\alpha_t=1-\tau_t,\beta_t=\tau_t,\kappa=4.$ We see the dilated interpolant estimates $p=.8$ much better than the non-dilated one.}
    \label{fig:enter-label}
\end{figure}

\subsection{Training a given feature on real data: MNIST}
Recall that in the background we mentioned that the analysis of \cite{biroli2024dynamicalregimesdiffusionmodels} shows that taking $\alpha_t=1-t$ and $\beta_t=t$ without any time-dilation gives an speciation time $t_s=1/\sqrt{d}.$ This then means that probability of each mode (given by $p$) can not be captured as $d\to\infty.$ Our analysis then shows that if we dilate time by stretching the interval $[0,\kappa/\sqrt{d}]$ to $[0, 1]$ and the interval $[\kappa/\sqrt{d}, 1]$ to $[1, 2],$ then we get accurate estimation of $p.$ 

When training diffusion models in practice, we first sample a batch of times $t_1,\cdots,t_k$ uniformly. We then draw $x^\mu_0\sim \mathcal{N}(0,\text{Id}_d),$ $x^\mu_1$ from our data distribution, and form a noisy sample $x^\mu_{t^\mu} = (1-{t^\mu})x^\mu_0 + {t^\mu}x^\mu_1$ for $\mu=1,\cdots,k.$ We finally train on the loss

\begin{equation}
    \mathcal{\hat{R}}(\theta) = \sum^k_{\mu=1}||f_{\theta}(x_{t^\mu}^{\mu},{t^\mu}) - x_1^\mu||^2.
    \label{eq:emp:loss2}
\end{equation}

where we took time as a parameter of the network as it is usually done in practice, as opposed to having a separate network for each time $t.$ 

\textit{The insight of our analysis is that instead of taking the batch of times uniformly, we can sample more times near the phase transition associated to a given feature, and in this way improve accuracy on that feature.}

For a given feature, we can find the times where that feature is learned using the U-Turn method (\cite{sclocchi2024phasetransitiondiffusionmodels}, \cite{biroli2024dynamicalregimesdiffusionmodels}). Consider a dataset where each sample corresponds to exactly one of finitely many classes. Examples of this are samples of the GM which correspond to one of two modes, or samples of MNIST which correspond to one of ten digits. The U-Turn then consists of starting with a sample from the data, run a backward diffusion model from time $t=1$ to $t=t_0,$ which noises the sample, and then run the forward diffusion model from time $t=t_0$ to $t=1$ with noise independent from the backward run. 

We are then interested in the probability that the sample before the backward and forward passes belongs to the same class as the sample after them. For $t_0 \approx 1,$ this probability is close to $1.$ For $t_0 \approx 0,$ this probability is close to the underlying probability of the diffusion model generating a sample of the given class. By running this for different $t_0,$ we can find at what times it is decided to what class the samples belong to. Having found those times, our goal is to have a model that generates samples for each class according to the probabilities that they appear in the dataset. We can then improve the accuracy of the model on this by training on these times.

As a simple example, we train a U-Net (see Appendix \ref{subsec:exp} for details) to parameterize the Variance Preserving SDE from \cite{song2021scorebased} to generate either the $0$ or $1$ digits from MNIST. The dataset we train on consists of $20\%$ $1$ digits and $80\%$ $0$ digits. We then measure how well is this model in generating samples that represent this asymmetry. The model is trained on approximately 7400 samples for 9 epochs, by sampling times in $[0,1]$ uniformly as described in the beginning of this section. We then generate $18500$ new samples running this model using 1000 discretization steps. \footnote{This amount of discretization steps is much larger than what is needed for MNIST, and we do it this way to make sure that the error is not coming from the integration of the SDE but from the training alone.} Among the $18500$ generated samples, $88.2\%$ are digits $0.$ (For determining this, we used a discriminator with $99.2\%$ accuracy on MNIST, see Appendix \ref{subsec:exp} for details.) 

We then test our proposed method. First, we determine at what time the digit that the sample represents is decided. We do this with the U-Turn method described above. Note that to do this, we use the model that we already trained. The results are in Figure \ref{fig:uturn}. We find that the times important for deciding the digit are early in the generation for $t\in[0.2, 0.6]$ and mostly concentrated on $t\in[0.3,0.5].$ 

We now train from scratch a model on $7400$ samples for 9 epochs as before, except that we do not sample the times uniformly. We instead sample times with probability $1/2$ uniformly in the interval $[0.3, 0.5]$ and with probability $1/2$ uniformly outside that interval. We then generate $18500$ new samples with this new model using $1000$ discretization steps, and find that $81.0\%$ are $0$s. We similarly consider sampling times with probability $1/2$ uniformly in the interval $[0.2, 0.6]$ and with probability $1/2$ outside that interval, generate samples, and find that $81.1\%$ are $0$s. This validates our hypothesis in the simple case of MNIST.

Although our theoretical analysis is for the probability flow ODE on the two-mode GM data distribution, this example on MNIST shows that the ideas developed here can be useful to the SDE generative models used in practice for real data.

\begin{figure}
    \centering
    \includegraphics[width=0.6\linewidth]{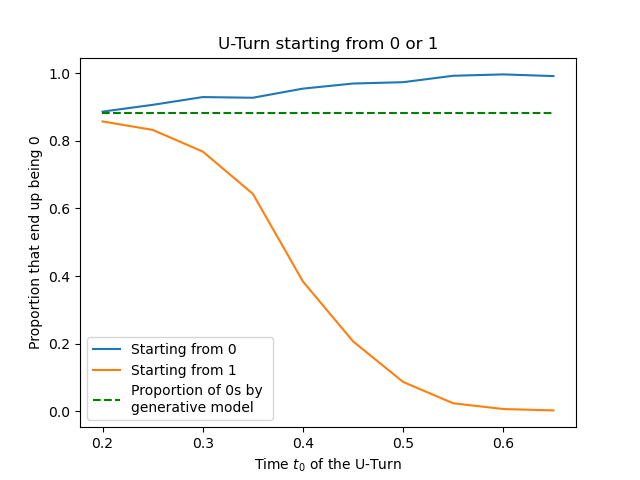}
    \caption{For $t_0\in [0.2, 0.65],$ we plot the proportion of $0$s that we get by doing the U-Turn at time $t_0$ starting from either $0$ or $1$ at time $t=1.$ On dashed green, we plot $y=.882$ which is the estimated proportion of $0$s that the diffusion model generates starting from noise.}
    \label{fig:uturn}
\end{figure}

\section{Acknowledgments}
The authors thank Eric Vanden-Eijnden and Hugo Cui for helpful discussions about this and previous work and Evan Dogariu for a discussion about Lemma \ref{lem:7}.


\bibliography{biblio}
\bibliographystyle{ICLR2025/iclr2025_conference}
\newpage
\appendix
\section{Proof of Proposition \ref{prp:char:gm}}
\label{app:prf}
To prove Proposition \ref{prp:char:gm}, we will use the following three Lemmas that follow directly from \cite{albergo2023stochasticinterpolantsunifyingframework} (Appendix A)
\begin{lem}
\label{lem:gen:gen} Let $a\sim \rho$ and $z\sim \mathcal{N}(0,\text{Id}_{d})$. The law of the interpolant $I_{t}=\alpha_{t}z+\beta_{t}a$ coincides with the law of the solution of the probability flow ODE
\begin{equation}
\begin{aligned}\dot{X}_{t} & =\frac{\alpha_{t}\dot{\alpha}_{t}+\sigma^2\beta_{t}\dot{\beta}_{t}}{\alpha_{t}^{2}+\sigma^2\beta_{t}^{2}}X_{t} 
+ \frac{\alpha_{t}(\alpha_{t}\dot{\beta}_{t}-\dot{\alpha}_{t}\beta_{t})}{\alpha_{t}^{2}+\sigma^2\beta_{t}^{2}} \mu\tanh\left(h+\frac{\beta_{t} \mu\cdot X_{t}}{\alpha_{t}^{2}+\sigma^2\beta_{t}^{2}}\right),\quad X_{0}\sim{\mathcal{N}}(0,\text{Id}_{d})\end{aligned}
\label{eq:b:gmm:2m}
\end{equation}
where $h$ is such that $e^h/(e^h+e^{-h})=p.$
\end{lem} 
\begin{lem}
\label{lem:gen:gen2} Let $a\sim p\mathcal{N}(\kappa, 1)+(1-p)\mathcal{N}(-\kappa, 1)$ and $z\sim \mathcal{N}(0,1)$. The law of the interpolant $I_{t}=\alpha_{t}z+\beta_{t}a$ coincides with the law of the solution of the probability flow ODE. In 
\begin{equation}
\begin{aligned}\dot{X}_{t} & =\frac{\alpha_{t}\dot{\alpha}_{t}+\beta_{t}\dot{\beta}_{t}}{\alpha_{t}^{2}+\beta_{t}^{2}}X_{t} 
+ \frac{\kappa\alpha_{t}(\alpha_{t}\dot{\beta}_{t}-\dot{\alpha}_{t}\beta_{t})}{\alpha_{t}^{2}+\beta_{t}^{2}} \tanh\left(h+\frac{\kappa\beta_{t} X_{t}}{\alpha_{t}^{2}+\beta_{t}^{2}}\right),\quad X_{0}\sim{\mathcal{N}}(0,1)\end{aligned}
\label{eq:b:gmm:2m2}
\end{equation}
where $h$ is such that $e^h/(e^h+e^{-h})=p.$
\end{lem} 

\begin{lem}
\label{lem:gen:gen3} Let $a\sim \mathcal{N}(\pm\mu,\sigma^2\text{Id}_{d})$ and $z\sim \mathcal{N}(0,\text{Id}_{d})$. The law of the interpolant $I_{t}=\alpha_{t}z+\beta_{t}a$ coincides with the law of the solution of the probability flow ODE
\begin{equation}
\begin{aligned}\dot{X}_{t} & =\frac{\alpha_{t}\dot{\alpha}_{t}+\sigma^2\beta_{t}\dot{\beta}_{t}}{\alpha_{t}^{2}+\sigma^2\beta_{t}^{2}}X_{t} 
\pm \frac{\alpha_{t}(\alpha_{t}\dot{\beta}_{t}-\dot{\alpha}_{t}\beta_{t})}{\alpha_{t}^{2}+\sigma^2\beta_{t}^{2}} \mu,\quad X_{0}\sim{\mathcal{N}}(0,\text{Id}_{d}).\end{aligned}
\label{eq:b:gmm:2m3}
\end{equation}
\end{lem} 

\begin{proof}[Proof of Proposition \ref{prp:char:gm}]
    \textbf{First phase.} We have $\tau_t = \frac{\kappa t}{\sqrt{d}}$ since $t\in [0,1].$ Plugging in $\alpha_t=1-\tau_t$ and $\beta_t=\tau_t$ into the velocity field from Lemma \ref{lem:gen:gen} yields 
    \begin{align}
        \dot{X}_{t} & = \frac{\kappa}{\sqrt{d}}\left(-X_t+\mu\tanh\left(h+\kappa t \frac{\mu\cdot X_{t}}{\sqrt{d}}\right)\right) + O\left(\frac{1}{d}\right).
        \label{eq:x:1st:phase}
    \end{align}
    We then have, with $\nu_t=\mu\cdot X_t/\sqrt{d},$
    \begin{align}
        \label{eq:vp:1st:mu}
        \dot{\nu}_{t} & = \kappa\tanh\left(h+\kappa t \nu_t\right) + O\left(\frac{1}{\sqrt{d}}\right).
    \end{align}
    Taking $d\to\infty$ yields the limiting ODE for $\nu_t$. From Lemma \ref{lem:gen:gen2}, we get that this the $1$-dimensional velocity field associated to the interpolant $I_t=\sqrt{1-t^2}z+ta$ that transports $z\sim{\mathcal{N}}(0,1)$ at $t=0$ to $a\sim p{\mathcal{N}}(\kappa,1) + (1-p){\mathcal{N}}(-\kappa, 1)$ at $t=1.$

    Let $X^\perp_t = X_t - \frac{\mu\cdot X_t}{{d}}\mu.$ We have from \eqref{eq:x:1st:phase}
    \begin{align}
        \label{eq:x_perp}
        \dot{X}^\perp_{t} & = -\frac{\kappa}{\sqrt{d}} X^\perp_t.
    \end{align}
    Since this is a linear ODE with initial condition Gaussian, we have 
    \begin{align}
        \dot{X}^\perp_{t} \sim \mathcal{N}\left(0, \sigma^2_t\text{Id}_{d-1}\right).
    \end{align}
    Further, \eqref{eq:x_perp} gives $\dot{X}^\perp_{t} = O({1/\sqrt{d}})$ meaning that for $t\in[0,1]$
    \begin{align}
        \lim_{d \to\infty} \sigma_t =1. 
    \end{align}
    \newline
    \textbf{Second phase.} For $t\in[1,2],$ we have $\tau_t = \left(1-\frac{\kappa}{\sqrt{d}}\right)(2t-1)+\frac{\kappa}{\sqrt{d}}.$ Again using Lemma \ref{lem:gen:gen}, we get 
    \begin{align}
        \dot X_t = \frac{-(2-t)+\sigma^2(t-1)}{(2-t)^2+\sigma^2 (t-1)^2}X_t + \frac{(2-t)\tanh\left(h + \frac{(t-1)\mu\cdot X_t + \kappa \frac{\mu\cdot X_t}{\sqrt{d}}}{(2-t)^2+\sigma^2 (t-1)^2}\right)}{(2-t)^2+\sigma^2 (t-1)^2}\mu + O\left(\frac{1}{\sqrt{d}}\right).
        \label{eq:ode:x}
    \end{align}
    Writing $\nu_t=\frac{\mu\cdot X_t}{\sqrt{d}},$ this implies
    \begin{align}
        \dot \nu_t = \frac{-(2-t)+\sigma^2(t-1)}{(2-t)^2+\sigma^2 (t-1)^2}\nu_t + \frac{(2-t)\sqrt{d}\tanh\left(h + \frac{(t-1)\sqrt{d}\nu_t + \kappa \nu_t}{(2-t)^2+\sigma^2 (t-1)^2}\right)}{(2-t)^2+\sigma^2 (t-1)^2} + O_d\left(1\right).
        \label{eq:ode:muu}
    \end{align}
    Let us calculate the initial condition for $\nu_t$ at $t=1.$ Write $a=sm + z$ where $p=\mathbb{P}(s=1) = 1-\mathbb{P}(s=-1)$ and $z\sim{\mathcal{N}}(0,\text{Id}_{d}).$ Then $$\frac{\mu\cdot I^\text{}_{t=1}}{\sqrt{d}} \stackrel{(d)}{=} Z + \kappa s + O\left(\frac{1}{\sqrt{d}}\right)$$ where $Z\sim{\mathcal{N}}(0,1).$ This means that for $\kappa$ large enough, then $|h|<\left|\frac{(t-1)\sqrt{d}\nu_t + \kappa \nu_t}{(2-t)^2+\sigma^2 (t-1)^2}\right|$ with high probability. This implies $\nu_t$ will not change sign during its trajectory, since whenever $\nu_t=o(\sqrt{d})$, the $\tanh$ term will dominate in \eqref{eq:ode:muu}. Hence, the following approximation is valid
    \begin{align}
        \tanh\left(h + \frac{(2t-1)\sqrt{d}\nu_t + \kappa \nu_t}{1+(\sigma^2-1)(2t-1)^2}\right) = \tanh\left(\sqrt{d}\nu_t\right) = \text{sgn}(\nu_t)
    \end{align}
    We then use this approximation in the ODEs for $X_t$ to get 
    \begin{align}
        \dot X_t = \frac{-(2-t)+\sigma^2(t-1)}{(2-t)^2+\sigma^2 (t-1)^2}X_t + \frac{(2-t)\text{sgn}(\nu_{t})}{(2-t)^2+\sigma^2 (t-1)^2}\mu + O\left(\frac{1}{\sqrt{d}}\right).
        \label{eq:ode:xx}
    \end{align}
    Let $M_t=\mu\cdot X_t/d.$ We get the induced equation
    \begin{align}
        \dot M_t = \frac{-(2-t)+\sigma^2(t-1)}{(2-t)^2+\sigma^2 (t-1)^2}M_t + \frac{(2-t)\text{sgn}(M_{t})}{(2-t)^2+\sigma^2 (t-1)^2} + O\left(\frac{1}{\sqrt{d}}\right).
        \label{eq:ode:xxx}
    \end{align} 
    From the analysis of the first phase we had 
    \begin{align}
        \nu_{1}\sim p\mathcal{N}(\kappa, 1) + (1-p)\mathcal{N}(-\kappa, 1).
    \end{align}
    We argued above that the sign of $\nu_t$ will be preserved for $t\in[1, 2]$ with probability going to $1$ as $\kappa$ tends to $\infty.$ This means that 
    \begin{align}
        M_2 = p^\kappa\delta_1 + (1-p^\kappa)\delta_{-1}
    \end{align}
    where $p^\kappa$ is such that  $\lim_{\kappa\to \infty}p^\kappa=p.$ 
    
    As in the first phase, we let $X^\perp_t = X_t - \frac{\mu\cdot X_t}{{d}}\mu.$ We have from \eqref{eq:ode:xx} that
    \begin{align}
        \label{eq:x_perp:2nd}
        \dot{X}^\perp_{t} & = \frac{-(2-t)+\sigma^2(t-1)}{(2-t)^2+\sigma^2 (t-1)^2}X^\perp_t + O\left(\frac{1}{\sqrt{d}}\right).
    \end{align}
    Since this is a linear ODE with initial condition Gaussian, we have 
    \begin{align}
        \dot{X}^\perp_{t} \sim \mathcal{N}\left(0, \sigma^2_t\text{Id}_{d-1}\right).
    \end{align}
    Under the change of variables $t(s)=s+1$ we get that the ODE becomes
    \begin{align}
        \label{eq:final:x_t3}
        \dot X^{\perp}_s= \frac{-(1-s)+\sigma^2s}
        {(1-s)^2+\sigma^2s^2}X^{\perp}_{s}.
    \end{align}
    By taking one coordinate $i\in\{1,\cdots, d-1\}$ of $X_s^\perp$ we get from Lemma \ref{lem:gen:gen3} that this is the velocity field associated with the interpolant $I_s=\sqrt{1-s^2}z+sa$ where $z\sim \mathcal{N}(0,1)$ is transported to $a\sim \mathcal{N}(0, \sigma^2)$ as desired. 
\end{proof}

\section{Derivations of learning results}
\subsection{Derivation of First Phase}
\label{sec:first_phase}
Let $t\in\left[0,1\right]$ so that 
\[
x_t^\mu = \left(1 - \frac{\kappa t}{\sqrt{d}}\right)x_0^\mu + \frac{\kappa t}{\sqrt{d}} x_1^\mu
\]
Consider a denoiser parametrized as
$$
f(x)=cx + u\tanh\left(b + \frac{w\cdot x}{\sqrt{d}}\right)
$$
We introduce the following overlaps which we assume to be of order 1 in $d$
\[
p_\eta^\mu = s^\mu \frac{z^\mu\cdot w}{d}, \quad
\omega = \frac{\mu\cdot w}{d}, \quad
r = \frac{\|w\|^2}{d}\quad 
q_\xi^\mu = s^\mu \frac{x_0^\mu\cdot u}{d},\quad 
q_\eta^\mu = s^\mu \frac{z^\mu\cdot u}{d}, \quad 
m = \frac{\mu\cdot u}{d}, \quad
q = \frac{\|u\|^2}{d}
\]
We note that 
\begin{align*}
    \phi^\mu = f(x_t^\mu)= \tanh\left(b + \kappa t  \sigma s^\mu p_\eta^\mu  + \kappa t \omega s^\mu + \sqrt{r} Z^\mu + o_d(1) \right)
\end{align*}
where $Z^\mu\sim\mathcal{N}(0,1).$ We now compute the loss
\begin{align*}
    \frac{1}{d}\sum_\mu\left\|x_1^\mu - f(x_t^\mu)\right\|^2 
    &=  \frac{1}{d}\sum_\mu\left\|x_1^\mu - c((1 - \kappa t/\sqrt{d})x_0^\mu + (\kappa t/\sqrt{d}) x_1^\mu) - u \phi^\mu\right\|^2 \\
    &=  \frac{1}{d}\sum_\mu\left\|x_1^\mu - c x_0^\mu - u \phi^\mu\right\|^2 + o_d(1)\\
    &= \sum_\mu 1+\sigma^2 +c^2+\frac{\|u\|^2}{d}(\phi^\mu)^2-2\left((\mu s^\mu + \sigma z^\mu)(1-c\kappa t/\sqrt{d}) - c x_0^\mu\right)\cdot\frac{u}{d}\phi^\mu + o_d(1)\\
    &= \sum_\mu 1+\sigma^2 + c^2 + q(\phi^\mu)^2 - 2(m + \sigma q_\eta^\mu - c q_\xi^\mu)s^\mu\phi^\mu + o_d(1)
\end{align*}
We follow the same style of calculation as \cite{cui2024analysislearningflowbasedgenerative} to compute the partion function. First we write the partition function 
\begin{align*}
    \mathcal{Z} 
    &= \int d\theta\, e^{-\frac{\gamma}{2}\hat R_t(\theta)} \\
    &=\int dc du dw db\, e^{-\frac{\gamma d}{2} \left(\sum_\mu 1+\sigma^2 + c^2 + \frac{\|u\|^2}{d}(\phi^\mu)^2 - 2( \frac{u\cdot \mu}{d}+ \sigma\frac{u\cdot z^\mu}{d} - c \frac{u\cdot x_0^\mu}{d})s^\mu\phi^\mu\right)-\frac{\gamma\lambda}{2} \|u\|^2 - \frac{\gamma\ell}{2}\|w\|^2 }
\end{align*}
Next we introduce overlaps into the integral and their conjugates by Dirac-Fourier, which we will denote as the vectors $\zeta$ and $\hat \zeta$ to simplify notation, and rearrange to integrate $u,w$
\begin{align*}
    \mathcal{Z} =& 
    \int dc db \, d\zeta d\hat \zeta\, 
    e^{d\left(\frac{1}{2}\hat q q+\hat m m +\frac{1}{2}\hat r r+\hat \omega \omega + \sum_{\mu=1}^n (q_\xi^\mu \hat q_\xi^\mu + q_\eta^\mu \hat q_\eta^\mu +p_\eta^\mu \hat p_\eta^\mu)) 
    - \frac{\gamma}{2} \sum_\mu (1+\sigma^2 + c^2 + q(\phi^\mu)^2 - 2( m + \sigma q_\eta^\mu - c q_\xi^\mu)s^\mu\phi^\mu )\right)}\\
    &\int du  e^{-\frac{\hat q + \gamma \lambda }{2} \|u\|^2  - u\cdot \left(\hat m \mu + \sum_{\mu=1}^n \hat q_\xi x_0^\mu + \hat q_\eta z^\mu \right) } \int dw e^{- \frac{\hat r + \gamma\ell}{2}\|w\|^2 - w\cdot\left(\hat \omega \mu + \sum_{\mu=1}^n\hat p_\eta z^\mu \right)}
\end{align*}
Next we evaluate the $u,w$ integrals to get
\begin{align*}
    &e^{d\left(\frac{1}{2}\log(\hat q + \gamma \lambda)+\frac{1}{2}\log(\hat r + \gamma \ell)+\frac{1}{2(\hat q + \gamma \lambda)} \frac{1}{d}\|\hat m \mu + \sum_{\mu=1}^n \hat q_\xi^\mu x_0^\mu + \hat q_\eta^\mu z^\mu \|^2 + \frac{1}{2(\hat r + \gamma\ell)}\frac{1}{d}\|\hat \omega \mu + \sum_{\mu=1}^n\hat p_\eta^\mu z^\mu\|^2  \right)}\\
    &= e^{d\left(\frac{1}{2}\log(\hat q + \gamma \lambda)+\frac{1}{2}\log(\hat r + \gamma \ell)+\frac{1}{2(\hat q + \gamma \lambda)}(\hat m^2 + \sum_{\mu=1}^n (\hat q_\xi^\mu)^2 + (\hat q_\eta^\mu)^2) + \frac{1}{2(\hat r + \gamma\ell)}(\hat \omega^2  + \sum_{\mu=1}^n(\hat p_\eta^\mu)^2) + O(1/\sqrt{d})  \right)}
\end{align*}
We now assume a sample-symmetry ansatz on the overlaps which means that $q_\xi^\mu=q_\xi$ for every $\mu$ are all equal, and the same for $ \hat{q}_\xi^\mu,q_\eta^\mu, \hat{q}_\eta^\mu,p_\eta^\mu, \hat{p}_\eta^\mu.$ We then take $d\to\infty$, rescale all conjugates with $\gamma$, change all conjugates signs except for $\hat{q}$ and $\hat{r}$ for cleaner equations, and take $\gamma \to \infty$. This gives us the following effective field (log partition function)
\begin{align*}
\log \mathcal Z(\mathcal D) = \text{extr}\bigg\{&-\frac{n}{2}\left(
-2(\sigma  q_\eta+ m - c q_\xi)\overline {\phi s}+c^2+q\overline{\phi^2}\right)\\
&+\frac{q\hat q}{2} +\frac{r\hat r}{2} - m\hat m -\omega \hat \omega- n (q_\xi\hat q_\xi +q_\eta\hat q_\eta +p_\eta \hat p_\eta)\\
&+\frac{\hat m^2+n(\hat q_\xi^2+\hat q_\eta^2)}{2(\lambda + \hat q)}+\frac{\hat\omega^2 + n \hat p_\eta^2}{2(\ell + \hat r)}
\bigg\}
\end{align*}
Taking gradients we get the following saddle-point equations

{
\setlength{\jot}{10pt} 
\renewcommand{\arraystretch}{1.5} 
\[
\left\{
\begin{array}{ll}
    (\sigma q_\eta + m-cq_\xi)\overline{\phi' s} - q\overline{\phi\phi'}=0\\
    c = -\overline{\phi s} q_\xi \\
    p_\eta =\frac{\hat p_\eta}{\ell + \hat r}\\
    \omega = \frac{\hat \omega}{\ell + \hat r}\\
    r = \frac{\hat \omega^2 + n \hat p_\eta^2}{(\ell + \hat r)^2}= \omega^2 + np_\eta^2\\
    q_\xi=\frac{\hat q_\xi}{\lambda + \hat q}\\
    q_\eta = \frac{\hat q_\eta}{\lambda + \hat q}\\
    m = \frac{\hat m}{\lambda + \hat q}\\
    q = \frac{\hat m^2 + n(\hat q_\xi^2 + \hat q_\eta^2)}{(\lambda + \hat q)^2}= m^2 + nq_\xi^2 + n q_\eta^2 
\end{array}
\right.
\quad
\left\{
\begin{array}{ll}
    \hat p_\eta = (\kappa t)((\sigma q_\eta + m - c q_\xi)\overline{\phi' s} - q\overline{\phi' \phi})=0\\
    \hat{\omega} = (n\kappa t)((\sigma q_\eta + m - c q_\xi)\overline{\phi' } - q\overline{\phi' \phi s})\\
    \hat r = -n((\sigma q_\eta + m - c q_\xi)\overline{Z \phi' s} - q\overline{Z\phi\phi'})/\sqrt{r}\\
    \hat q_\xi = -c\overline{\phi s} \\
    \hat q_\eta = \sigma \overline{\phi s}\\
    \hat{m} = n\overline{\phi s}\\
    \hat q = n\overline {\phi^2}\\
\end{array}
\right.
\]
}

Combining the equations for $c$ and $q_\xi$ we get that $c = c \overline{\phi s}^2/(\lambda + n\overline{\phi^2})$. We now argue that $c=0$ almost surely, since otherwise $\overline{\phi}=0$ on a non-zero measure, implying $b=\omega=0$, which then results in all overlaps being zero, giving a suboptimal log partition function. This can be seen more explicitly by noting that the log partition function is zero for $c\neq 0$, but for $c=0$ it is instead 

\begin{align*}
    \log \mathcal Z(\mathcal D) = \text{extr}_{\omega, b}\left\{n\frac{ \overline{\phi s}^2(\sigma^2 + n)}{2\left(\lambda + n \overline{\phi^2}\right)} - \frac{1}{2}\ell \omega^2 \right\}
\end{align*}

which has positive values for example at $\omega=0,b\neq 0$. The above formulation is also useful for solving for the overlaps numerically.

\subsubsection{Argument for Corollary \ref{cor:sta}}
\label{sec:cor1}
We now focus on $n\to \infty$ and on verifying that $\omega=\kappa t, b = \tanh^{-1}(\bar s)$ is a solution. We will need the following preliminary claims.

\begin{lem} Let $\phi^\mu=\tanh(b + \kappa t \omega s^\mu + \omega  Z^\mu)$. If $\omega = \kappa t$ and $b=\tanh^{-1}(\overline{s})$ then $\overline{\phi - s}=0$
\end{lem}
\begin{proof}
    Let $\phi_\pm^\mu = \phi^\mu|_{s^\mu=\pm 1}$ and $\phi_0 = \phi^\mu|_{s^\mu=0}$. Then 
    \begin{align*}
        \overline{\phi - s} 
        &= \overline{p}\overline{(\phi_+ - 1)} + (1-\overline{p})\overline{(\phi_- + 1)}\\
        &= \int dz\, e^{-\frac{z^2}{2}}\left \{ \overline{p} (\phi_+ - 1)+(1-\overline{p}) (\phi_- + 1)\right\} \\
        &= \int dz\, e^{-\frac{z^2+(\kappa t)^2}{2}}\left \{ e^{\kappa t z} \overline{p} (\phi_0 - 1)+e^{-\kappa t z}(1-\overline{p}) (\phi_0 + 1)\right\} \quad z\to z\mp \kappa t \\
    \end{align*}
    Finally note that the integrad is zero for all $z$ if
    \begin{align*}
        \phi_0 = \frac{e^{\kappa t z}\overline{p} - e^{-\kappa t z} (1 - \overline{p})}{e^{\kappa t z}\overline{p}+(1- \overline{p})e^{-\kappa tz}}=\tanh\left(\tanh^{-1}(\overline{s}) + \kappa t z \right)
    \end{align*}
    which occurs for $\omega = \kappa t$ and $b=\tanh^{-1}(\overline{s})$. 
\end{proof}
\begin{cor}\label{cor:phibar}
    Let $\phi^\mu=\tanh(b + \kappa t \omega s^\mu + \omega Z^\mu)$. If $\omega = \kappa t$ and $b=\tanh^{-1}(\overline{s})$ then for any function $g$ where $\overline{g(Z\pm\kappa t)(\phi \mp 1)}$ exist we have
    \[
    \overline{g(Z+s\kappa t)(\phi - s)}=0
    \]
    In particular,
    \[
    \overline{\phi-s}=\overline{\phi(\phi - s)}=\overline{\phi'(\phi-s)}=\overline{(Z +s \kappa t )\phi'(\phi - s)} = 0
    \]
\end{cor}

Solving for $q_\eta,m,q_\xi,q$ and plugging into the equation for $b$ we get
\begin{align*}
    (\lambda + n\overline{\phi^2})\,\overline{\phi's}=n\overline{\phi s}\, \overline{\phi \phi'}
\end{align*}
Taking $n\to\infty$, to leading order in $n$ the equality becomes $(\overline{\phi^2})(\overline{\phi' s})-(\overline{\phi'\phi})(\overline{\phi s})=0$ which holds by \Corref{cor:phibar}.

Using the independence of $s^\mu, Z^\mu$ and taking the limit of infinitely many $Z^\mu$, we can use Stein's lemma to rewrite the $\hat r$ equation as 
\begin{align*}
      \hat r &=-n((\sigma  q_\eta + m-c q_\xi)\overline{\phi'' s} - q\overline{(\phi\phi')'}) 
\end{align*}
Plugging in $q_\eta,m,q_\xi,q$ in gives
\begin{align*}
    \hat r (\lambda + n\overline{\phi^2})^2 
    &= - n((\lambda + n\overline{\phi^2})(\sigma^2 + n) \overline{\phi'' s} \, \overline{\phi s}  - n(\sigma^2 +  n)\overline{\phi s}^2\,\overline{(\phi\phi')'})
\end{align*}
Plugging $q_\eta,m,q_\xi,q$ into equations for $\omega$ and $\hat \omega$ gives
\begin{align*}
    \omega(\ell + \hat r)(\lambda + n\overline{\phi^2})^2 
    &= (n\kappa t)((\lambda + n\overline{\phi^2})(\sigma^2  + n)(\overline{\phi' } \, \overline{\phi s}) - n(\sigma^2 + n) \overline{\phi s}^2 \overline{\phi'\phi s} )
\end{align*}
For large $n$ we get
\begin{align*}
    -\omega ((\overline{\phi^2}) (\overline{\phi'' s})(\overline{\phi s})  - (\overline{\phi s}^2)(\overline{(\phi\phi')'}) &= (\kappa t)((\overline{\phi^2})(\overline{\phi' }) (\overline{\phi s})) - (\overline{\phi s}^2 )(\overline{\phi'\phi s}))\\
    \omega  &= (\kappa t)\frac{((\overline{\phi' })  - (\overline{\phi'\phi s}))}{-( (\overline{\phi'' s})  - (\overline{(\phi\phi')'}))}\\
    &=(\kappa t)\frac{\overline{(\phi - \tfrac{1}{2}\phi^2 s)' }  }{-\overline{(\phi - \tfrac{1}{2}\phi^2 s)''s}}
\end{align*}
Finally note that
\begin{align*}
    \overline{(\phi - \tfrac{1}{2}\phi^2 s)'} + \overline{(\phi - \tfrac{1}{2}\phi^2 s)''s} = -\frac{1}{2\kappa t}\overline{\left[(\phi - s)^2\right]'\left(s\kappa t + Z\right)} = 0
\end{align*}
applying \Corref{cor:phibar}.

\color{black}
\subsection{Derivation of Second Phase}
\label{sec:res2}

We now consider times $t\in \left[1,2\right]$ which means we have 
$$x_t^\mu = (2-t)\left(1-\frac{\kappa}{\sqrt{d}}\right)x_0^\mu + \left(\frac{\kappa}{\sqrt{d}} + \left(1-\frac{\kappa}{\sqrt{d}}\right)\left(t-1\right)\right)x_1^\mu.$$
We change variables to $\tau= t-1$ and consider $\tau\in \left[0, 1\right]$ so that 
$$x_\tau^\mu = (1-\tau)\left(1-\frac{\kappa}{\sqrt{d}}\right)x_0^\mu + \left(\frac{\kappa}{\sqrt{d}} + \left(1-\frac{\kappa}{\sqrt{d}}\right)\tau\right)x_1^\mu.$$
We compute the loss for a single datapoint, defining $\nu^\mu = s^\mu \phi\left(w\cdot x_\tau^\mu+b\right)$
\begin{align*}
    &\frac{1}{d}\left\|x_1^\mu - c\left((1-\tau)\left(1-\frac{\kappa}{\sqrt{d}}\right)x_0^\mu + \left(\frac{\kappa}{\sqrt{d}} + \left(1-\frac{\kappa}{\sqrt{d}}\right)\tau\right)x_1^\mu\right) - u \phi\left(w\cdot x_\tau^\mu+b\right)\right\|^2 \\
    &= \frac{1}{d}\left\|x_1^\mu - c((1-\tau)x_0^\mu + \tau x_1^\mu) - u s^\mu \nu^\mu\right\|^2+o_d(1)\\ 
    &= \frac{1}{d}\left\|(1-c\tau)(\sigma z^\mu + s^\mu \mu) - c(1-\tau)x_0^\mu - u s^\mu \nu^\mu\right\|^2+o_d(1)\\
    &= (1+\sigma^2)(1-c\tau)^2 + c^2(1-\tau)^2 + \frac{\|u\|^2}{d} -\frac{2 s^\mu \nu^\mu}{d} u \cdot ((1-c\tau)(\sigma z^\mu + s^\mu \mu) - c(1-\tau)x_0^\mu) +o_d(1)\\
    &= (1+\sigma^2)(1-c\tau)^2 + c^2(1-\tau)^2 + q -2\nu^\mu(1-c\tau)(\sigma q_\eta^\mu + m) + 2\nu^\mu c(1-\tau)q^\mu_\xi +o_d(1)
\end{align*}
where we defined the overlaps
\[
    q = \frac{\|u\|^2}{d}, \quad q_\eta^\mu = s^\mu \frac{u\cdot z^\mu}{d}, \quad q_\xi^\mu = s^\mu\frac{u\cdot x^\mu_0}{d}, \quad m = \frac{u\cdot\mu}{d}
\]
\[
    \quad p_\eta^\mu = s^\mu \frac{w\cdot z^\mu}{d}, \quad p_\xi^\mu = s^\mu\frac{w\cdot x^\mu_0}{d}, \quad \omega = \frac{w\cdot\mu}{d}.
\]

We also have 
\begin{align*}
    \nu^\mu = \phi\left(s^\mu w \cdot x_\tau^\mu+s^\mu b\right) &= \tanh\left((1-\tau)s^\mu w\cdot x_0^\mu + \left(\tau+\frac{\kappa}{\sqrt{d}}\right)s^\mu w\cdot (\sigma z^\mu + s^\mu \mu)+s^\mu b\right)\\
    &= \tanh\left(d\left((1-\tau)p_\xi^\mu + \left(\tau+\frac{\kappa}{\sqrt{d}}\right)(\sigma p_\eta^\mu + \omega)\right)+s^\mu b\right)\\
    &\asymp \text{sign}\left(\sqrt{d}\left((1-\tau)p_\xi^\mu + \tau(\sigma p_\eta^\mu + \omega)\right)+\kappa(\sigma p_\eta^\mu + \omega)\right)
\end{align*}

This gives the following 

\begin{align*}
    \log \mathcal Z(\mathcal{D}) = \text{extr}\Bigg\{&
        -\frac{n}{2}\left((1+\sigma^2)(1-c\tau)^2 + c^2(1-\tau)^2 + q -2(1-c\tau)(\sigma q_\eta + m)\overline \nu + 2 c(1-\tau)q_\xi\overline \nu \right) 
        \\&+\frac{q\hat q}{2} - m\hat m - n (q_\xi\hat q_\xi +q_\eta\hat q_\eta ) +\frac{\hat m^2+n(\hat q_\xi^2+\hat q_\eta^2)}{2(\lambda + \hat q)}
    \Bigg\}
\end{align*}
Taking gradients we get the following saddle-point equations
{
\setlength{\jot}{10pt} 
\renewcommand{\arraystretch}{1.5} 
\[
\left\{
\begin{array}{ll}
    q_\xi = \frac{\hat q_\xi}{\lambda + \hat q} = \frac{c(1-\tau)\overline{\nu}}{\lambda + n} \\
    q_\eta = \frac{\hat q_\eta}{\lambda + \hat q} = \frac{(1-c\tau)\overline{\nu}}{\lambda + n} \\
    m = \frac{\hat m}{\lambda + \hat q}=\frac{n(1-c\tau)\overline{\nu}}{\lambda + n} \\
    q = \frac{\hat m^2 + n(\hat q_\xi^2 + \hat q_\eta^2)}{(\lambda + \hat q)^2} = m^2 + nq_\xi^2 + n\sigma^2 q_\eta^2 \\
    c=\frac{(1+\sigma^2)\tau-\tau(\sigma q_\eta+m)\overline \nu-(1-\tau)q_\xi\overline \nu}{(1-\tau)^2+(1+\sigma^2)\tau^2}
\end{array}
\right.
\quad
\left\{
\begin{array}{ll}
    \hat q_\xi =  c(1-\tau)\overline{\nu} \\
    \hat q_\eta = \sigma (1-c\tau)\overline{\nu} \\
    \hat{m} = n(1-c\tau)\overline{\nu} \\
    \hat q = n \\
\end{array}
\right.
\]
}

\begin{align*}
    c &= \frac{(1+\sigma^2) \tau (\lambda + n) - \overline{\nu}^2 \tau (\sigma+n)}{ (\lambda+n)((1-\tau^2)+(1+\sigma^2)\tau^2)+\overline{\nu}^2\left((1-\tau)^2-\tau^2(\sigma+n)\right)}
\end{align*}

Corollary \ref{cor:3} simply follows from taking the $n\to\infty$ limit of this equations.

Lastly, we now argue that we can take $\overline{\nu}=1$ without loss of generality. If we assume a sample symmetric ansatz for $p_\eta^\mu=p_\eta, p_\xi^\mu=p_\xi$, then $\overline{\nu}$ can either be $\pm 1$ depending on the sign of argument. Noting that $q,c$ are unchanged while $q_\eta,q_\xi,m,\hat q_\eta,\hat q_\xi,\hat m$ flip sign, we can conclude that the log partition function is invariant so $\overline{\nu}=1$.

The characterizations of the learned parameters can be used to evaluate the MSE as a function of $t$, which we now describe, in the limit of $d\to\infty$ and then $n\to \infty$. For the first and second phase we obtain
\[
\text{mse}_{\text{train}} = \text{mse}_{\text{test}}= \begin{cases}
    \sigma^2 + (1-\overline{\phi^2}) & t\in [0,1]\\
   \sigma^2(1-c\tau)^2 + c^2(1-\tau)^2 & t \in[1,2]
\end{cases}
\]
At $t=0$, $\phi=\tanh(b) = 2(p - 1/2)$ hence the MSE is $\sigma^2  + 4p(1-p)$. At $t=1$ we have $c=0$ hence the MSE is $\sigma^2$, while at $t=2$ we have $c=1$ so the MSE is $0$. 

\section{Arguments for generation}
\label{gen:app}
Combining equations \ref{eq:dec:b} and \ref{eq:exact:b} gives the exact velocity field
\begin{equation}
    b_t(x) = \left(\dot\beta_t-\frac{\dot\alpha_t}{\alpha_t}\beta_t\right)\left(\frac{\beta_t\sigma^2}{\alpha_t^2 + \sigma^2\beta_t^2} x + \frac{\alpha_t^2}{\alpha_t^2 + \sigma^2\beta_t^2}\mu \tanh\left(\frac{\beta_t}{\alpha_t^2 + \sigma^2\beta_t^2} \mu \cdot x + h\right)\right) + \frac{\dot\alpha_t}{\alpha_t}x.
\end{equation}
where $\alpha_t=1-\tau_t$ and $\beta_t=\tau_t$ with $\tau_t$ from \eqref{eq:time_dil}. Let $\hat\theta_t$ denote any overlap from the first phase (see \eqref{first:ov}) in the limit of $d\to\infty$ but for finite $n,$ where $\theta_t$ denotes the corresponding overlap with $d\to\infty$ and then $n\to \infty.$ From Results \ref{res:first} and \ref{res:second} and their Corollaries \ref{cor:sta} and \ref{cor:3}, we have that $|\hat \theta_t - \theta_t|=O_n(1/n)$ for all overlaps.


Since $X_t-\hat{X}_t$ is contained in span$(u_t, \eta)$ which is in turn contained in span$(\mu, \eta, \xi)$, it suffices to show that, after dividing by $d$, the projections of $X_t-\hat{X}_t$ onto $\mu,$ $\eta,$ and $\xi$ are $O(1/n)$ to show that $\frac{1}{d}\|X_t-\hat{X}_t\|$ is $O(1/n).$
\subsection{Argument for Result \ref{res:exact:vs:learned}}
First, we note that as described in the paragraph above the statement of Result \ref{res:exact:vs:learned}, we have that since in the first phase $q=m^2+nq_\eta^2$ from Result \ref{res:first} we get for $t\in [0,1]$
\begin{align*}
    \lim_{d\to\infty} \frac{\|u_t\|^2}{d} = \lim_{d\to\infty}\left(\frac{\mu\cdot u_t}{d}\right)^2 + \left(\frac{\eta\cdot u_t}{d}\right)^2
\end{align*}
also since $q= m^2 + nq_\xi^2 + n q_\eta^2$ in the second phase, we get that for $t\in [1,2]$
\begin{align*}
    \lim_{d\to\infty} \frac{\|u_t\|^2}{d} = \lim_{d\to\infty}\left(\frac{\mu\cdot u_t}{d}\right)^2 + \left(\frac{\eta\cdot u_t}{d}\right)^2 + \left(\frac{\xi\cdot u_t}{d}\right)^2
\end{align*}
where $\eta =\sigma  \sum_{\mu=1}^n z^\mu$ and $\xi =\sum_\mu s^\mu x_0^\mu$
which implies that for any $w\in\text{span}(\mu,\eta,\xi)^\perp$ with $\|w\|_2=1$ we have
\begin{align*}
    \lim_{d\to\infty} \frac{w\cdot(X_2-\hat{X}_2)}{\sqrt{d}}=0.
\end{align*}
\subsubsection{First phase}
We focus on $t\in[0,1]$ and define
\begin{align*}
    &\epsilon^m_t = \frac{1}{\sqrt{d}} \mu \cdot (X_t - \hat{X}_t),\quad
    \epsilon^\eta_t = \frac{1}{\sigma^2 n\sqrt{d}} \eta \cdot (X_t - \hat{X}_t),\\
    &\delta_t = \frac{\beta_t}{\alpha_t^2 + \sigma^2 \beta_t^2},\quad
    \gamma_t = \frac{\alpha_t^2}{\alpha_t^2 + \sigma^2 \beta_t^2},\\
    &M_t = \frac{\mu \cdot X_t}{\sqrt{d}},\quad
    Q_t^\eta = \frac{\eta \cdot X_t}{\sigma^2 n\sqrt{d}}.
\end{align*}
We have
\begin{align*}
    \dot \epsilon_t^m &= \frac{1}{\sqrt{d}}\mu\cdot (\dot X_t - \dot{\hat {X_t}})\\
    =& \frac{1}{\sqrt{d}}\mu \cdot  (b_t(X_t) - \hat b_t(\hat X_t)) \\
    =& \frac{1}{\sqrt{d}}\mu \cdot \left(\dot \beta - \frac{\dot \alpha}{\alpha}\beta\right)\bigg( \sigma^2 \delta_t (X_t-\hat X_t) + (c_t-\sigma^2 \delta_t ) X_t + (\gamma_t\mu - u_t) \tanh(\delta_t \mu\cdot X_t + h) \\
    &+ u_t\left(\tanh(\delta_t \mu\cdot X_t + h) - \tanh(w_t\cdot X_t + b_t)\right) \bigg)+ \frac{\dot \alpha}{\alpha }\frac{1}{\sqrt{d}}\mu \cdot\left(X_t-\hat X_t \right)\\
    =& \left(\dot \beta - \frac{\dot \alpha}{\alpha}\beta\right)\bigg( \sigma^2 \delta_t \epsilon^m_t + (c_t-\sigma^2 \delta_t ) M_t + \sqrt{d}(\gamma_t - m_t) \tanh(\delta_t \mu\cdot X_t + h) \\
    &+ \sqrt{d}m_t\left(\tanh(\delta_t \mu\cdot X_t + h) - \tanh(w_t\cdot \hat X_t/\sqrt{d} + b_t)\right) \bigg)+ \frac{\dot \alpha}{\alpha }\epsilon^m_t\\
    =& \kappa \left(1 - t\frac{\dot \alpha}{\alpha}\right)\bigg( \frac{\delta_t}{\sqrt{d}} \epsilon^m_t + (c_t-\sigma^2 \delta_t ) \frac{M_t}{\sqrt{d}} + (\gamma_t - m_t) \tanh(\delta_t \mu\cdot X_t + h) \\
    &+ m_t\left(\tanh(\delta_t \mu\cdot X_t + h) - \tanh(w_t\cdot \hat X_t/\sqrt{d} + b_t)\right) \bigg)+ \frac{\dot \alpha}{\alpha }\epsilon^m_t
\end{align*}
We now focus on the $\tanh$
\begin{align*}
    &\left|\tanh(\delta_t \mu\cdot X_t + h) - \tanh(w_t\cdot \hat X_t/\sqrt{d} + b_t)\right|\\
    &\leq \left|\left(\delta_t \mu - \frac{w_t}{\sqrt{d}}\right)\cdot X_t\right| + \left|\frac{w_t}{\sqrt{d}} \cdot (X_t-\hat{X}_t)\right|+|h-b_t|\\
    &\leq \left|\left(\delta_t \mu - \frac{w_t}{\sqrt{d}}\right)\left(\frac{\mu\mu^T}{d}+\frac{\eta\eta^T}{\sigma^2 n^2d}\right) X_t\right| +\left| \frac{w_t}{\sqrt{d}}\left(\frac{\mu\mu^T}{d}+\frac{\eta\eta^T}{\sigma^2 n^2d}\right) (X_t-\hat{X}_t)\right|+\left|h-b_t\right|\\
    &\leq |\left(\sqrt{d}\delta_t - \omega_t\right)M_t| + |\left(\delta_t Z - p^\eta_t\right)Q_t|  + |\omega_t\epsilon^m_t| +|p_t^\eta\epsilon^\eta_t| + |h-b_t|\\
    &\leq \omega_t|\epsilon^m_t| + O\left(\frac{1}{n}\right)+ O\left(\frac{1}{\sqrt{d}}\right).
\end{align*}

Coming back to the ODE for $\dot \epsilon^m_t,$ we get with high probability
\begin{align*}
    |\dot \epsilon^m_t| = \kappa\gamma_t\omega_t|\epsilon^m_t| + O\left(\frac{1}{n}\right)+ O\left(\frac{1}{\sqrt{d}}\right).
\end{align*}
Since $\kappa\gamma_t\omega_t=\Theta_{n,d}(1)$ for $t\in[0,1],$ we get that with high probability $$\epsilon^m_{t=1} = O\left(\frac{1}{n}\right)+ O\left(\frac{1}{\sqrt{d}}\right).$$

By performing a similar computation for the ODE for $\epsilon^\eta_t,$ we get that with high probability $$\epsilon^\eta_{t=1} = O\left(\frac{1}{n}\right)+ O\left(\frac{1}{\sqrt{d}}\right).$$

\subsubsection{Second phase}
We now turn to $t\in[1,2]$ and define
\begin{align*}
    &\zeta^m_t = \frac{1}{{d}} \mu \cdot (X_t - \hat{X}_t),\quad
    \zeta^\eta_t = \frac{1}{\sigma^2 n{d}} \eta \cdot (X_t - \hat{X}_t)\quad
    \zeta^\xi_t = \frac{1}{n{d}} \xi \cdot (X_t - \hat{X}_t).
\end{align*}

With high probability, we have the following ODEs hold
\begin{align}
\frac{d}{dt} \zeta^m &= \left( \beta(t) c_t + \frac{\dot{\alpha}(t)}{\alpha(t)} (1 - c_t \beta(t)) \right) \zeta^m + O\left( \frac{1}{n} \right), \\
\frac{d}{dt} \zeta^\eta &= \left( \beta(t) c_t + \frac{\dot{\alpha}(t)}{\alpha(t)} (1 - c_t \beta(t)) \right) \zeta^\eta + O\left( \frac{1}{n} \right),\\
\frac{d}{dt} \zeta^\xi &= \left( \beta(t) c_t + \frac{\dot{\alpha}(t)}{\alpha(t)} (1 - c_t \beta(t)) \right) \zeta^\xi + O\left( \frac{1}{n} \right). 
\end{align}

from the initial condition $\zeta_1^m,\zeta_1^\eta=O(\tfrac{1}{\sqrt{d}})+O(\tfrac{1}{n}),\zeta_1^\xi=0.$ This yields
\[
\zeta_2^m , \zeta_2^\eta, \zeta_2^\xi = O(\tfrac{1}{n})+O(\tfrac{1}{\sqrt{d}})
\]
\subsection{Argument for Corollary \ref{cor:5}}
By Proposition \ref{prp:char:gm}, we know that 
\begin{align*}
    \lim_{\kappa\to\infty} \lim_{d\to\infty} \frac{\mu\cdot X_2}{d}\sim p\delta_1+(1-p)\delta_{-1}.
\end{align*}
By Result \ref{res:exact:vs:learned}, we get that 
\begin{align*}
    \lim_{n\to\infty} \lim_{d\to\infty} \frac{\mu\cdot (\hat{X}_2 - X_2)}{d} = 0.
\end{align*}
Combining the last two equations gives the first claim from the Corollary.

Fix $w\perp\mu$, $\|w\|=1.$ Again by Proposition \ref{prp:char:gm}, we have that 
\begin{align*}
    \lim_{d\to\infty}\frac{w\cdot X_2}{\sqrt{d}} \sim \mathcal{N}(0,\sigma^2).
\end{align*}
Also, Result \ref{res:exact:vs:learned} gives that 
\begin{align*}
    \lim_{n\to\infty} \lim_{d\to\infty}\frac{w\cdot (\hat{X}_2 - X_2)}{\sqrt{d}} = 0.
\end{align*}
which combined with the previous equation gives the second claim from the Corollary.

\section{Experimental details}
\label{subsec:exp}
The model used for the MNIST experiment consists of a U-Net architecture (\cite{DBLP:journals/corr/RonnebergerFB15}), consisting of four downsampling and four upsampling blocks with two layers per block and output channels of 128, 128, 256, and 512, respectively. Attention mechanisms are integrated into the third downsampling block and the second upsampling block to enhance feature representation at multiple scales. The training of the denoiser is described in the main text. We then use this denoiser to estimate the score and run the Variance Preserving SDE (see equation (11) in \cite{song2021scorebased}.)

For the discriminative model, we use the MNIST digit classification model by \cite{knight2022mnist} available on Hugging Face which achieves an accuracy of $99.1\%$ on MNIST classification.

As a sanity check, we show non-cherry-picked samples generated by the three models we considered in Figure \ref{fig:mnist_digits}. 
\begin{figure}
    \centering
    \begin{subfigure}{0.49\linewidth}
        \centering
        \includegraphics[width=\linewidth]{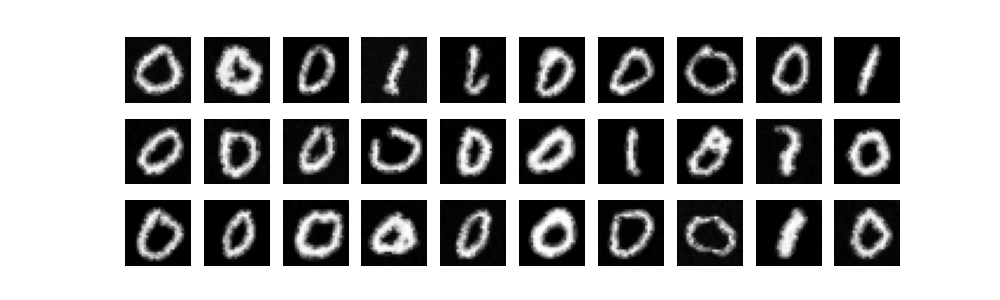}
        \caption{Training times with \\ prob. $1/2$ on $[.2, .6]$}
    \end{subfigure}
    \hfill
    \begin{subfigure}{0.49\linewidth}
        \centering
        \includegraphics[width=\linewidth]{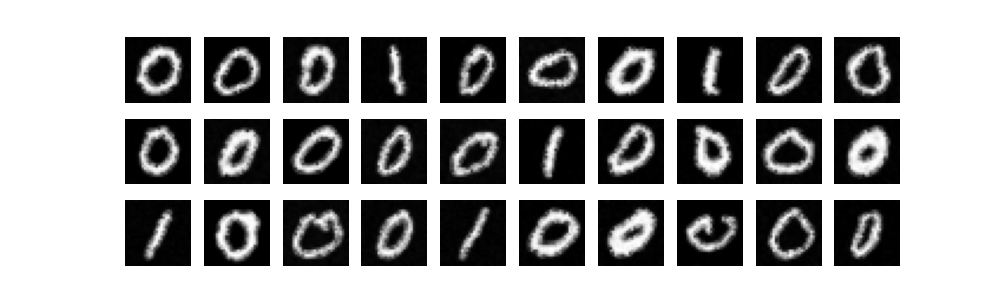}
        \caption{Training times with \\ prob. $1/2$ on $[.3, .5]$}
    \end{subfigure}
    \begin{subfigure}{0.49\linewidth}
        \centering
        \includegraphics[width=\linewidth]{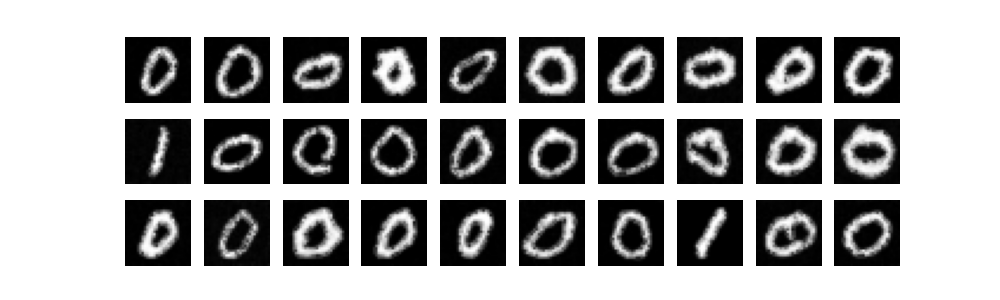}
        \caption{Training times \\ uniform on $[0, 1]$}
    \end{subfigure}
    \caption{Non-cherry-picked samples from the three generative models considered. \textbf{(a)} Samples from the VP SDE, where the times for training are drawn with probability $1/2$ uniformly from $[.2, .6]$ and with probability $1/2$ uniformly outside. \textbf{(b)} Same as left panel except that with probability $1/2$ training times are sampled from $[.3, .5].$ \textbf{(c)} Samples from the VP SDE with training times that are uniform in $[0, 1].$}
    \label{fig:mnist_digits}
\end{figure}

\section{General time dilation formula}
\label{app:gen:dil}
In this section, we generalize the time dilation formula from equation \ref{eq:time_dil} for a Gaussian mixture with more than two modes. Although the arguments in Results 1 and 2 only hold for the two-mode GM, the fact that a more general time dilation formula exists suggests that these results could be extended to the GM with more than two modes. \\

Consider $\mu=\sum_{i=1}^{m}p_{i}\mathcal{N}\left(r_{i},\sigma^{2}\text{I}\right)$ where $r_{i}\in\mathbb{R}^{d}$ and $|r_{i}|$ goes to infinity with $d,$but $m,p_{i},\sigma^{2}$ are constant with respect to $d.$ If $X_{t}$ is the generative model associated with the interpolant $I_{t}=(1-t)z+ta$ where $z\sim\mathcal{N}(0,\text{I})$ and $a\sim\mu$ (as we do in equation \eqref{eq:ode:gen:1} in the main text) then $X_{t}$ estimates $p_{i}$ at times of the order $1/|r_{i}|.$ We show this in Proposition \ref{prp:2} below by arguing that it is only at times of order $1/|r_{i}|$ that the denoiser associated to $r_{i}\cdot X_{t}/|r_{i}|$ is nontrivial. Hence, to estimate $p_{i}$ we require a time dilation $\tau_{t}$ such that there exists $a$ and $b$ with $b-a=\Theta_{d}(1)$ where 
\begin{align}
\tau_{t} & =\Theta_{d}\left(\frac{1}{|r_{i}|}\right)\,\,\text{for \,\,\,}t\in[a,b].\label{eq:cond}
\end{align}

We specify next a time dilation that for every $i$ would ensure that Equation \ref{eq:cond} is fulfilled. Assume $|r_{1}|\leq|r_{2}|\leq\cdots\leq|r_{m}|$, let $n=m+1$ and let $\kappa>0.$ Then
\begin{equation}
\tau_{t}=\begin{cases}
\frac{\kappa nt}{|r_{m}|} & \text{if }t\in[0,1/n]\\
\frac{\kappa(nt-1)}{|r_{m-1}|}+\frac{\kappa}{|r_{m}|} & \text{if }t\in[1/n,2/n]\\
\cdots\\
\frac{\kappa(nt-(m-1))}{|r_{1}|}+\kappa\left(\frac{1}{|r_{2}|}+\cdots+\frac{1}{|r_{m}|}\right) & \text{if }t\in[(m-1)/n,m/n]\\
\left(1-\kappa\left(\frac{1}{|r_{1}|}+\cdots+\frac{1}{|r_{m}|}\right)\right)t+\kappa\left(\frac{1}{|r_{1}|}+\cdots+\frac{1}{|r_{m}|}\right) & \text{if }t\in[m/n,1]
\end{cases}\label{eq:form}
\end{equation}

Then we have that $p_{i}$ is learned when $t\in[(m-i)/n,(m-i+1)/n]$ and the $\sigma^{2}$ will be learned when $t\in[m/n,1],$ giving rise to $m+1$ different phases. In the special case of $|r_{i}|=|r_{i+1}|,$ both $p_{i}$ and $p_{i+1}$ will already be learned in $[(m-i)/n,(m-i+1)/n]$ so that the phase on the interval $[(m-i+1)/n,(m-i+2)/n]$ is unnecessary. Taking this consideration into account when using the general formula in equation \ref{eq:form} for the two-mode GM gives the time dilation formula from equation \ref{eq:time_dil}. The only difference is that the time dilation here maps $[0,1]$ to $[0,1]$ and the one in equation \ref{eq:time_dil} maps $[0,1]$ to $[0,2].$ \\

\begin{prop}
    \label{prp:2}
    Let $\mu=\sum_{i=1}^{m}p_{i}\mathcal{N}\left(r_{i},\sigma^{2}\text{I}\right)$ where $r_{i}\in\mathbb{R}^{d}$ and $|r_{i}|=\omega_{d}(1).$ Consider the interpolant $I_{t}=(1-t)z+ta$ where $z\sim\mathcal{N}(0,\text{Id})$ and $a\sim\mu.$ Let $X_{t}$ be the generative model associated to $I_{t}$ as in equation \eqref{eq:ode:gen:1}. Then $X_{t}$ learns the $p_{i}$ at times $\Theta_{d}\left(1/|r_{i}|\right).$
\end{prop} 
\begin{proof}
Fix $i$. Let $m_{t}=r_{i}\cdot I_{t}/|r_{i}|.$ We have $m_{t}\stackrel{d}{=}(1-t)Z+t|r_{i}|m$ where $Z\sim\mathcal{N}(0,1)$ and $m=r_{i}\cdot a/|r_{i}|^{2}=\Theta_{d}(1).$ Let $\nu_{t}=r_{i}\cdot X_{t}/|r_{i}|.$ By Lemma \ref{lem:6}, $\nu_{t}$ obeys the self-consistent ODE 
\begin{equation}
\dot{\nu}_{t}=\frac{\nu_{t}}{t}-\frac{\eta_{t}(\nu_{t})}{t}\label{eq:nu_ode}
\end{equation}
 where $\eta_{t}(\nu)$ is the denoiser for $\nu_{t}$
\[
\eta_{t}(\nu)=\mathbb{E}[Z|m_{t}=\nu]=\mathbb{E}[Z|(1-t)Z+t|r_{i}|m=\nu].
\]

By Lemma \ref{lem:7}, since $|r_{i}|=\omega_{d}(1)$, the only times where this denoiser is nontrivial are $t=\Theta_{d}\left(1/|r_{i}|\right).$ We note that to estimate $p_{i}$ we need to estimate $\nu_{t}$, which requires spending a constant length of time in the nontrivial times of the ODE in equation \ref{eq:nu_ode}, which are the nontrivial times for the denoiser. Indeed, $p_{i}$ is learned on that interval, and if the length of that interval goes to $0$ as $d$ goes to infinity, we cannot estimate $p_{i}.$\\
\end{proof}

\begin{lem}
\label{lem:6}
Let $\mu=\sum_{i=1}^{m}p_{i}\mathcal{N}\left(r_{i},\sigma^{2}\text{I}\right)$ where $r_{i}\in\mathbb{R}^{d}.$ Consider the interpolant $I_{t}=(1-t)z+ta$ where $z\sim\mathcal{N}(0,\text{Id})$ and $a\sim\mu.$ Let $X_{t}$ be the generative model associated to $I_{t}$ from Lemma 5. Fix $i$ and let $m_{t}=r_{i}\cdot I_{t}/|r_{i}|$ and $\nu_{t}=r_{i}\cdot X_{t}/|r_{i}|.$ Then with $\eta_{t}(\nu)=\mathbb{E}[Z|m_{t}=\nu]$ we have 
\[
\dot{\nu}_{t}=\frac{\nu_{t}}{t}-\frac{\eta_{t}(\nu_{t})}{t}
\]
\end{lem}
\begin{proof}
We have from Appendix A, \cite{albergo2023stochasticinterpolantsunifyingframework} that the velocity field $b_{t}(x)$ associated with $I_{t}=(1-t)z+ta$ where $a\sim\sum_{i=1}^{m}p_{i}\mathcal{N}\left(r_{i},\text{Id}\right)$ can be written explicitly as 
\begin{align*}
\frac{\sum_{i=1}^{m}p_{i}\left(r_{i}+\frac{\dot{c_{t}}}{2c_{t}}(x-tr_{i})\right)\mathcal{N}(x\mid tr_{i},c_{t}\text{I})}{\sum_{i=1}^{m}p_{i}\mathcal{N}(x\mid tr_{i},c_{t}\text{I})} & =\frac{\sum_{i=1}^{m}p_{i}\left(r_{i}+\frac{\dot{c_{t}}}{2c_{t}}(x-tr_{i})\right)e^{\frac{2tr_{i}\cdot x-t^{2}|r_{i}|^{2}}{2((1-t)^{2}+t^{2})}}}{\sum_{i=1}^{m}p_{i}e^{\frac{2tr_{i}\cdot x-t^{2}|r_{i}|^{2}}{2((1-t)^{2}+t^{2})}}}\\
 & =\frac{\dot{c_{t}}}{2c_{t}}x+\frac{\sum_{i=1}^{m}p_{i}\left(1-\frac{\dot{c_{t}}}{2c_{t}}t\right)e^{\frac{2tr_{i}\cdot x-t^{2}|r_{i}|^{2}}{2((1-t)^{2}+t^{2})}}}{\sum_{i=1}^{m}p_{i}e^{\frac{2tr_{i}\cdot x-t^{2}|r_{i}|^{2}}{2((1-t)^{2}+t^{2})}}}r_{i}
\end{align*}

where $c_{t}=(1-t)^{2}+t^{2}.$ The denoiser $\eta_{t}(x)=\mathbb{E}[z|I_{t}=x]$ is
\begin{align}
\eta_{t}(x)=x-tb_{t}(x) & =\left(1-\frac{\dot{c_{t}}}{2c_{t}}\right)x-\frac{\sum_{i=1}^{m}p_{i}\left(t-\frac{\dot{c_{t}}}{2c_{t}}t^{2}\right)e^{\frac{2tr_{i}\cdot x-t^{2}|r_{i}|^{2}}{2((1-t)^{2}+t^{2})}}}{\sum_{i=1}^{m}p_{i}e^{\frac{2tr_{i}\cdot x-t^{2}|r_{i}|^{2}}{2((1-t)^{2}+t^{2})}}}r_{i}.\label{eq:den_x}
\end{align}

Fix $i$ and let $m_{t}=r_{i}\cdot I_{t}/|r_{i}|$ and $\nu_{t}=r_{i}\cdot X_{t}/|r_{i}|.$ Since $\dot{X}_{t}=b(X_{t}),$ we get that 
\[
\dot{\nu}_{t}=\frac{r_{i}\cdot b(X_{t})}{|r_{i}|}=\frac{\nu_{t}}{t}-\frac{1}{t}\frac{r_{i}\cdot\eta_{t}(X_{t})}{|r_{i}|}=\frac{\nu_{t}}{t}-\frac{\eta_{t}(\nu_{t})}{t},
\]

where the denoiser for the $\nu_{t}$ is defined as $\eta_{t}(\nu)=\mathbb{E}[Z|m_{t}=\nu].$ The last step in the displayed equality follows since from equation \ref{eq:den_x} we get that $r_{i}\cdot\eta_{t}(x)/|r_{i}|$ depends on $x$ only through $\nu_{t}.$
\end{proof}

\begin{lem}
\label{lem:7}
Let $Z\sim\mathcal{N}(0,1)$ and $M\sim\mu$. Then for fixed $\gamma>0$ we have that as $d\to\infty$
\begin{align*}
\mathbb{E}[Z|Z+d^{-\gamma}M=x] & \to x\\
\mathbb{E}[Z|Z+d^{\gamma}M=x] & \to\mathbb{E}[Z]=0
\end{align*}
\end{lem}
\begin{proof}
Let $f_{Z,X}(z,x)$ be the joint density of $Z$ and $X=Z+d^{-\gamma}M$ and $f_{Z,M}(z,m)$ the joint density of $Z$ and $M.$ We note that $f_{Z,X}(z,x)=f_{Z,M}(z,d^{\gamma}(x-z))=f_{Z}(z)f_{M}(d^{\gamma}(x-z))$
\begin{align*}
\mathbb{E}[Z|Z+d^{-\gamma}M=x] & =\frac{\int zf_{Z,X}(z,x)dz}{\int f_{Z,X}(z,x)dz}\\
 & =\frac{\int zf_{Z}(z)f_{M}(d^{\gamma}(x-z))dz}{\int f_{Z}(z)f_{M}(d^{\gamma}(x-z))dz}\\
 & =\frac{\int zf_{Z}(z)d^{\gamma}f_{M}(d^{\gamma}(x-z))dz}{\int f_{Z}(z)d^{\gamma}f_{M}(d^{\gamma}(x-z))dz}\\
 & \to x
\end{align*}

where the last step follows since $d^{\gamma}f_{M}(d^{\gamma}z)$ is an approximation to the identity. The other limit follows similarly.
\end{proof}
\end{document}

%% file: ICLR2025/math_commands.tex
\usepackage{amsmath,amsfonts,bm}

\def\eqref#1{equation~\ref{#1}}

\def\1{\bm{1}}

\DeclareMathAlphabet{\mathsfit}{\encodingdefault}{\sfdefault}{m}{sl}
\SetMathAlphabet{\mathsfit}{bold}{\encodingdefault}{\sfdefault}{bx}{n}

\newcommand{\E}{\mathbb{E}}



%% file: macro.tex
\usepackage{amsmath, amsthm, amssymb, graphicx, color, hyperref}

\theoremstyle{plain}

\newtheorem{lem}{Lemma}
\newtheorem{prop}{Proposition}
\newtheorem{res}{Result}
\newtheorem{cor}{Corollary}
\theoremstyle{definition}

\theoremstyle{remark}

\def\Corref#1{Corollary~\ref{#1}}

%% file: main.bbl
\begin{thebibliography}{20}
\providecommand{\natexlab}[1]{#1}
\providecommand{\url}[1]{\texttt{#1}}
\expandafter\ifx\csname urlstyle\endcsname\relax
  \providecommand{\doi}[1]{doi: #1}\else
  \providecommand{\doi}{doi: \begingroup \urlstyle{rm}\Url}\fi

\bibitem[Albergo et~al.(2023)Albergo, Boffi, and Vanden-Eijnden]{albergo2023stochasticinterpolantsunifyingframework}
Michael~S. Albergo, Nicholas~M. Boffi, and Eric Vanden-Eijnden.
\newblock Stochastic interpolants: A unifying framework for flows and diffusions, 2023.
\newblock URL \url{https://arxiv.org/abs/2303.08797}.

\bibitem[Ambrogioni(2023)]{ambrogioni2023statistical}
Luca Ambrogioni.
\newblock The statistical thermodynamics of generative diffusion models.
\newblock \emph{arXiv preprint arXiv:2310.17467}, 2023.

\bibitem[Benton et~al.(2024)Benton, Bortoli, Doucet, and Deligiannidis]{benton2024nearlydlinearconvergencebounds}
Joe Benton, Valentin~De Bortoli, Arnaud Doucet, and George Deligiannidis.
\newblock Nearly $d$-linear convergence bounds for diffusion models via stochastic localization, 2024.
\newblock URL \url{https://arxiv.org/abs/2308.03686}.

\bibitem[Biroli \& Mézard(2023)Biroli and Mézard]{Biroli_2023}
Giulio Biroli and Marc Mézard.
\newblock Generative diffusion in very large dimensions.
\newblock \emph{Journal of Statistical Mechanics: Theory and Experiment}, 2023\penalty0 (9):\penalty0 093402, September 2023.
\newblock ISSN 1742-5468.
\newblock \doi{10.1088/1742-5468/acf8ba}.
\newblock URL \url{http://dx.doi.org/10.1088/1742-5468/acf8ba}.

\bibitem[Biroli et~al.(2024)Biroli, Bonnaire, de~Bortoli, and Mézard]{biroli2024dynamicalregimesdiffusionmodels}
Giulio Biroli, Tony Bonnaire, Valentin de~Bortoli, and Marc Mézard.
\newblock Dynamical regimes of diffusion models, 2024.
\newblock URL \url{https://arxiv.org/abs/2402.18491}.

\bibitem[Chen et~al.(2023)Chen, Chewi, Lee, Li, Lu, and Salim]{chen2023probabilityflowodeprovably}
Sitan Chen, Sinho Chewi, Holden Lee, Yuanzhi Li, Jianfeng Lu, and Adil Salim.
\newblock The probability flow ode is provably fast, 2023.
\newblock URL \url{https://arxiv.org/abs/2305.11798}.

\bibitem[Cui et~al.(2024)Cui, Krzakala, Vanden-Eijnden, and Zdeborová]{cui2024analysislearningflowbasedgenerative}
Hugo Cui, Florent Krzakala, Eric Vanden-Eijnden, and Lenka Zdeborová.
\newblock Analysis of learning a flow-based generative model from limited sample complexity, 2024.
\newblock URL \url{https://arxiv.org/abs/2310.03575}.

\bibitem[Gatmiry et~al.(2024)Gatmiry, Kelner, and Lee]{gatmiry2024learningmixturesgaussiansusing}
Khashayar Gatmiry, Jonathan Kelner, and Holden Lee.
\newblock Learning mixtures of gaussians using diffusion models, 2024.
\newblock URL \url{https://arxiv.org/abs/2404.18869}.

\bibitem[Ho et~al.(2020)Ho, Jain, and Abbeel]{ho2020denoisingdiffusionprobabilisticmodels}
Jonathan Ho, Ajay Jain, and Pieter Abbeel.
\newblock Denoising diffusion probabilistic models, 2020.
\newblock URL \url{https://arxiv.org/abs/2006.11239}.

\bibitem[Kingma \& Ba(2015)Kingma and Ba]{diederik2014adam}
Diederik~P. Kingma and Jimmy Ba.
\newblock Adam: A method for stochastic optimization.
\newblock \emph{arXiv preprint arXiv:1412.6980}, 2015.
\newblock URL \url{https://arxiv.org/abs/1412.6980}.

\bibitem[Knight(2022)]{knight2022mnist}
Farley Knight.
\newblock {MNIST Digit Classification Model}.
\newblock \url{https://huggingface.co/farleyknight/mnist-digit-classification-2022-09-04}, 2022.
\newblock Accessed: [September 28, 2024].

\bibitem[Li \& Chen(2024)Li and Chen]{li2024critical}
Marvin Li and Sitan Chen.
\newblock Critical windows: non-asymptotic theory for feature emergence in diffusion models.
\newblock \emph{arXiv preprint arXiv:2403.01633}, 2024.

\bibitem[Lou et~al.(2024)Lou, Meng, and Ermon]{lou2024discretediffusionmodelingestimating}
Aaron Lou, Chenlin Meng, and Stefano Ermon.
\newblock Discrete diffusion modeling by estimating the ratios of the data distribution, 2024.
\newblock URL \url{https://arxiv.org/abs/2310.16834}.

\bibitem[Montanari(2023)]{montanari2023samplingdiffusionsstochasticlocalization}
Andrea Montanari.
\newblock Sampling, diffusions, and stochastic localization, 2023.
\newblock URL \url{https://arxiv.org/abs/2305.10690}.

\bibitem[Raya \& Ambrogioni(2023)Raya and Ambrogioni]{raya2023spontaneoussymmetrybreakinggenerative}
Gabriel Raya and Luca Ambrogioni.
\newblock Spontaneous symmetry breaking in generative diffusion models, 2023.
\newblock URL \url{https://arxiv.org/abs/2305.19693}.

\bibitem[Ronneberger et~al.(2015)Ronneberger, Fischer, and Brox]{DBLP:journals/corr/RonnebergerFB15}
Olaf Ronneberger, Philipp Fischer, and Thomas Brox.
\newblock U-net: Convolutional networks for biomedical image segmentation.
\newblock \emph{CoRR}, abs/1505.04597, 2015.
\newblock URL \url{http://arxiv.org/abs/1505.04597}.

\bibitem[Sclocchi et~al.(2024)Sclocchi, Favero, and Wyart]{sclocchi2024phasetransitiondiffusionmodels}
Antonio Sclocchi, Alessandro Favero, and Matthieu Wyart.
\newblock A phase transition in diffusion models reveals the hierarchical nature of data, 2024.
\newblock URL \url{https://arxiv.org/abs/2402.16991}.

\bibitem[Sohl-Dickstein et~al.(2015)Sohl-Dickstein, Weiss, Maheswaranathan, and Ganguli]{pmlr-v37-sohl-dickstein15}
Jascha Sohl-Dickstein, Eric Weiss, Niru Maheswaranathan, and Surya Ganguli.
\newblock Deep unsupervised learning using nonequilibrium thermodynamics.
\newblock In Francis Bach and David Blei (eds.), \emph{Proceedings of the 32nd International Conference on Machine Learning}, volume~37 of \emph{Proceedings of Machine Learning Research}, pp.\  2256--2265, Lille, France, 07--09 Jul 2015. PMLR.
\newblock URL \url{https://proceedings.mlr.press/v37/sohl-dickstein15.html}.

\bibitem[Song \& Ermon(2020)Song and Ermon]{song2020generativemodelingestimatinggradients}
Yang Song and Stefano Ermon.
\newblock Generative modeling by estimating gradients of the data distribution, 2020.
\newblock URL \url{https://arxiv.org/abs/1907.05600}.

\bibitem[Song et~al.(2021)Song, Sohl-Dickstein, Kingma, Kumar, Ermon, and Poole]{song2021scorebased}
Yang Song, Jascha Sohl-Dickstein, Diederik~P Kingma, Abhishek Kumar, Stefano Ermon, and Ben Poole.
\newblock Score-based generative modeling through stochastic differential equations.
\newblock In \emph{International Conference on Learning Representations}, 2021.
\newblock URL \url{https://openreview.net/forum?id=PxTIG12RRHS}.

\end{thebibliography}
